\documentclass{article}

\usepackage[colorlinks=true]{hyperref}
\usepackage{url}
\usepackage{graphicx}
\usepackage{amsmath, amssymb, amsfonts}

 \usepackage{amsthm}
\newtheorem{theorem}{Theorem}
\newtheorem{lemma}[theorem]{Lemma}

\newtheorem{remark}[theorem]{Remark}

\theoremstyle{definition}

\let\oldproofname=\proofname
\renewcommand{\proofname}{\rm\bf{\oldproofname}}

\def\cX{\mathcal X}

\def\RR{\mathbb R}

\def\bE{\mathbf E}

\def\bx{\mathbf x}

\def\bc{\mathbf c}
\def\bt{\mathbf t}
\def\bK{\mathbf K}

\def\L2p{{L^2_{\rho_{\!_\cX}}}}
\def\H{\mathcal H}
\def\calX{\mathcal X}

\def\calH{\mathcal H}

\def\bw{{\mathbf w}}

\def\dsum{\displaystyle\sum}

\def\bv{\mathbf v}

\graphicspath {{figures/}}

\begin{document}

\title{Bias Correction for Regularized Regression and
its Application in Learning with Streaming Data}

\author{
Qiang Wu \\
Department of Mathematical Sciences\\
Middle Tennessee State University\\
Murfreesboro, TN 37132 \\
\texttt{qwu@mtsu.edu} \\
}

\date{}

\maketitle

\begin{abstract}
We propose an approach to reduce the bias of
ridge regression and regularization kernel network.
When applied to a single data set
the new algorithms have comparable learning performance
with the original ones.
When applied to incremental learning with block wise streaming data
 the new algorithms are more efficient due to bias reduction.
 Both theoretical characterizations and
 simulation studies are used to verify the effectiveness of these new algorithms.
\end{abstract}

\section{Introduction}
\label{sec:intro}

As modern technologies allow to collect data much easily,
the size of data sets is growing fast in both dimensionality and
number of instances. This makes big data become ubiquitous in many fields
and draw great attentions of researchers in  recent years.
In the statistics and machine learning context, many traditional data processing
tools and techniques become inviable due to the big size of the data.
New models and computational techniques are required. This had driven
the resurgence of the research in online learning and the use of parallel computing.

Online learning deals with streaming data.
Online algorithms update the knowledge
incrementally as new data come in.
The streaming data could be instance wise or block wise.
The instance wise streaming data could be processed
as block wise data. This may be preferred
in some particular application domains.
For instance, in the dynamic pricing problems
(see e.g. \cite{wang2014close, agrawal2014dynamic})
the price is usually not updated each time
when an instance of sales information becomes available,
because customers may not like the price changing too constantly.
When processing block wise streaming data, a base algorithm is applied to each
incoming data block and a coupling method is then used to update the knowledge
by combining the knowledge from the past blocks and the incoming block;
see e.g. \cite{chavent2014sliced, he2011incremental}.
In statistical learning theory where the knowledge is usually represented by
a target function, the simplest way to couple the information is to use the
average of the functions learnt from different blocks.
In learning with big data, the divide and conquer algorithm \cite{zhang2013divide} divides the whole data set into smaller subsets,
applies a base learning algorithm on each subset,
and takes the average of the learnt functions from  all subsets
as the target function for prediction purpose.
It is computationally efficient because the second stage
could be implemented via parallel computing.
Although the divide and conquer algorithm is different from the aforementioned online learning with block wise
 streaming data, they clearly share the same spirit --
 a base algorithm for a single data block is required and the average of the outputs from this algorithm over
 multiple data blocks is used.  A natural problem arising from these two frameworks is the choice of the
 base learning algorithm for a single data set. As an algorithm is efficient and optimal for a single data set,
 it is not necessarily efficient and optimal for learning with block wise data.

In this paper we focus on the regression learning problem where
a set of observations are collected for  $p$ predictors and a scalar response
variable. Assume they are  linked by $$ y_i = f^*(\bx_i)+\epsilon_i,\qquad i=1,2,\ldots, n,$$
where $\bx_i \in\RR^p$, $y_i\in\RR,$ and $\epsilon_i$  is the zero-mean noise.
The target is to recover the unknown true model $f^*$ as accurate as possible
to understand the impact of predictors and predict the response on unobserved data.
The ordinary least square (OLS) is the most traditional and well developed method. It assumes
a linear model and estimates the coefficients by minimizing the squared error between the responses
and the predictions.  The OLS estimator requires the inverse of the covariance matrix of the explanatory variables
and could be numerically unstable if the covariance matrix is singular or has very large condition number.
A variety of regularization approaches have been introduced to overcome the numerical instability and/or
for some other purposes (e.g. sparsity).
Typical regularized methods include  ridge regression \cite{hoerl1970ridge1,hoerl1970ridge2}, LASSO \cite{Tib},
 elastic net \cite{ZH} and many others.
The nonlinear extension of ridge regression could be implemented by regularization kernel network \cite{EPP}.
The data are first mapped to a feature space. Then  a linear model is built in the feature space which,
when projected back to the original space, becomes a nonlinear model.

Although different regularization techniques have different properties, they share some common features.
The estimators obtained from regularized regression are usually biased. The regularization is helpful to improve the
computational stability and reduce the variance. By trading off the bias and variance, regularization schemes
may lead to smaller prediction errors than unbiased estimators. Therefore, regularization theory has become
an important topic in the statistical learning context.

Regularization algorithms, such as the ridge regression, regularization kernel network, and support vector machines,
have been successful in a variety of regression and classification applications. However,
 they may be suboptimal  when they serve as base algorithms in learning
 with block wise streaming data or in the divide and conquer algorithm.
 When there are many data blocks, the regularization algorithm may provide good estimator for each data block.
 By coupling the estimators together,  the variance usually shrinks when more and more data blocks are taken into account.
 But  the bias may not shrink and prevent the algorithm to achieve optimal performance.
 To overcome this difficulty, adjustments are required to remove or reduce the bias of the algorithm.

 In the paper, we will propose an approach to
 correct the bias of ridge regression and regularization kernel network.
 The resulted two new algorithms and their properties
 will be described in Section \ref{sec:barr} and Section \ref{sec:barkn}.
 Their theoretical properties are proved in 
 Section \ref{sec:pf1} and Section \ref{sec:pf2}, respectively.
 In Section \ref{sec:block} we discuss
 why the new algorithms are effective
 in learning with block wise data.
 In Section \ref{sec:sim} simulations are used to illustrate their effectiveness
 from an empirical aspect. We close by Section \ref{sec:conclusion} with conclusions and discussions.

\subsection{Related Work}

The idea of bias correction has long history in statistics.  
For instance, bias correction to maximum likelihood estimation dates at least back to 1950s \cite{quenouille1956notes}
and a variety method were proposed later on; see e.g. \cite{mclachlan1980note,schaefer1983bias,firth1993bias,cribari2002nearly}.
Bias reduction to kernel density estimators was studied in 
\cite{breiman1977variable, abramson1982bandwidth, jones1995simple, chung2011likelihood}.
Bias correction to nonparametric estimation was studied in \cite{hall1990bias, park1997simple, yao2012bias}.

The existence of bias in ridge regression and its impact on statistical inference 
has been noticed since its invention \cite{hoerl1970ridge1, obenchain1977classical}.
In high dimensional linear models where the dimension greatly exceeds the sample size,
bias correction method was introduced in \cite{buhlmann2013statistical}  to correct the projection bias,
the difference of the true regression coefficient and its projection in the subspace spanned by the sample,
which appears because the sample cannot span the whole Euclidian space $\RR^p$ as $n\ll p.$
In \cite{zhang2014confidence, buhlmann2013statistical, javanmard2014confidence}, 
projection bias correction was introduced to LASSO in high dimensional linear models.
The purpose of projection bias is to construct accurate $p$ values to facilitate 
accurate statistical inference such as hypothesis testings and confidence intervals.
It seems the bias caused by regularization has minimal impact for this purpose.

As for regularization kernel network, its predictive consistency has been extensively studied in the literature;
see e.g. \cite{bousquet2002stability, zhang2003leave, devito2005model, wu2006learning, bauer2007regularization, 
caponnetto2007optimal, smale2007learning,  sun2009note, steinwart2009optimal} 
and many references therein. Its application was also extensively explored and shown successful in 
many problem domains.  But to my best knowledge, 
the idea of bias correction to improve this algorithm is novel.
 Note that bias reduction for regularized regression does not improve the learning performance on a single data set, 
 as illustrated in Section \ref{sec:sim}.
It is worthy of exploration because of its effectiveness in learning with streaming data or distributed regression.

 \section{Bias correction for ridge regression}
 \label{sec:barr}

 In linear regression, the response variable is assumed to linearly depend on the explanatory variables, i.e.
 $$ y_i = \bw^\top \bx_i + b + \epsilon_i$$
 with some $\bw\in\RR^p$ and $b\in\RR.$
 Ridge regression  minimizes the penalized mean squared prediction error  on the observed data
 $$\frac 1n \sum_{i=1}^n (y_i-\bw^\top \bx_i -b)^2 + \lambda \|\bw\|^2,$$
where $\lambda>0$ is the regularization parameter used to trade off the fitting error and model complexity.
Let $\bar\bx$ denote the sample mean of $\bx_i$'s and $\bar y$ be the sample mean of $y_i$'s. Denote by
 $\tilde X  = [\bx_1-\bar\bx, \ \bx_2-\bar\bx, \ \ldots, \bx_n-\bar\bx]$ the centered data matrix for the explanatory variables
and $\tilde Y = (y_1-\bar y, y_2-\bar y, \ldots, y_n-\bar y)^\top$ the vector of centered response values.
Then the sample covariance matrix of $\bx$ is $\hat\Sigma =\frac 1 n \tilde X \tilde X^T$
and the solution to the ridge regression is given by
$$\hat \bw = \frac 1 n  \left(\lambda I + \hat\Sigma\right)^{-1} \tilde X^\top \tilde Y$$  and $\hat b = \bar y - \hat\bw \bar \bx.$
Here and in the sequel, $I$ denotes
the identity matrix (or the identity operator).

 The solution $\hat\bw$ is a biased estimator for $\bw.$
 Define $$b(n, \lambda) = \|\bE[\hat\bw] - \bw\|$$ and $$v(n,\lambda) = \bE\left[\|\hat\bw-\bE[\hat\bw]\|^2\right].$$
 Then $b(n, \lambda)$ is the Euclidian norm of the bias and $v(n,\lambda)$ is the variance. The mean squared error is
 given by $$\hbox{mse}(n,\lambda) =\bE\left[\|\hat\bw -\bw\|^2\right] = b^2(n, \lambda) + v(n, \lambda).$$

 Denote by $\Sigma$ the covariance matrix of the explanatory variables $\bx$. Let
 $\sigma_i$ be the eigenvalues of $\Sigma$  and $\bv_i$ the corresponding  eigenvectors.
 Then $$\Sigma = \sum_{i=1}^ p \sigma_i \bv_i\bv_i^\top.$$
 The vectors  $\bv_i$ are the principal components.
 The following theorem characterizes the bias and variance of ridge regression.

 \begin{theorem} \label{thm:ridge}
 If $\bx$ is bounded and $\bE[\epsilon^2]<\infty$, then
 \begin{enumerate}
 \item[{\rm (i)}]  $\bE[\hat\bw]$ converges $ (\lambda I + \Sigma)^{-1}\Sigma \bw$ as $n\to\infty.$
 \item[{\rm (ii)}] If $\bw = \sum_{i=1}^p c_i \bv_i$, then
 $$\lim\limits_{n\to \infty} b(n, \lambda) 
 =\sqrt{\sum_{i=1}^p \frac {\lambda^2 c_i^2}{(\lambda + \sigma_i)^2}}.$$
 \item[{\rm (iii)}] $v(n, \lambda) = O(\frac 1 {n \lambda^2}).$
 \end{enumerate}
 \end{theorem}

 Without loss of generality, we assume the eigenvalues are in decreasing order, i.e., $\sigma_i\ge \sigma_{i+1}$.
 Then $\frac{\lambda}{\lambda+\sigma_i}$ is increasing.  Theorem \ref{thm:ridge} tells that, for a fixed $\lambda$,
 the bias of ridge regression will be small  if the true model $\bw$ heavily depends on the first several principle components.
 Conversely, if $\bw$ heavily depends on the last several principle components, the bias will be large.

 According to Theorem \ref{thm:ridge} (i),  the asymptotic bias
 of ridge estimator is $-\lambda (\lambda I + \Sigma)^{-1} \bw$.
 If we can subtract it from the ridge regression estimator,
 we are able to obtain an asymptotically unbiased estimator
 $\hat\bw + \lambda (\lambda I + \Sigma)^{-1} \bw$.
 However, this is not computationally feasible because both $\Sigma$ and $\bw$ are unknown.
 Instead, we propose to replace $\Sigma$ by its sample version $\hat\Sigma$ and
 $\bw$ by the ridge estimator $\hat\bw$. The resulted new estimator,
 which we call bias corrected ridge regression estimator, becomes
\begin{equation}\label{eq:wa}
\hat\bw^\sharp = \hat\bw + \lambda(\lambda I + \hat\Sigma)^{-1} \hat\bw.
\end{equation}

Since the bias correction uses an estimated quantity,
this new estimator is still biased. But the bias is reduced.
Let $$b^\sharp(n, \lambda) = \|\bE[\hat\bw^\sharp] - \bw\|$$ and $$v^\sharp(n,\lambda) = \bE\left[\|\hat\bw^\sharp-\bE[\hat\bw^\sharp]\|^2\right].$$
We have the following conclusion.

 \begin{theorem} \label{thm:barr}
  If $\bx$ is bounded and $\bE[\epsilon^2]<\infty$, then
 \begin{enumerate}
 \item[{\rm (i)}] $\bE[\hat\bw^\sharp]$ converges to
 $(\lambda I + \Sigma)^{-2} (2\lambda I +\Sigma)\Sigma\bw$ as $n\to\infty.$
 \item[{\rm (ii)}] If $\bw = \sum_{i=1}^p c_i \bv_i$, then
 $$\lim\limits_{n\to \infty} b^\sharp(n, \lambda) 
 =\sqrt{\sum_{i=1}^p \frac {\lambda^4 c_i^2}{(\lambda + \sigma_i)^4}}.$$
 \item[{\rm (iii)}] $v^b(n, \lambda) = O(\frac 1 {n \lambda^2 }).$
 \end{enumerate}
 \end{theorem}

Since $\frac \lambda{\lambda+\sigma_i}<1$,
the asymptotic bias $\hat \bw ^\sharp$ is smaller.
The bias reduction could be significant
if the true model depends only on the first several principle components.
We also remark that, although $v(n, \lambda)$ and $v^\sharp(n, \lambda)$
are of the same order,
$v^\sharp(n, \lambda)$ is found slightly larger in simulations.
The overall performance, as measured by the mean squared error,
 of these two estimators is  comparable when used in learning with a single data set.

\section{Proofs of Theorem \ref{thm:ridge} and Theorem \ref{thm:barr}}
\label{sec:pf1}

In this section we proveTheorem \ref{thm:ridge} and Theorem \ref{thm:barr}. 
To this end, we first introduce several useful lemmas.
In our analysis, we will deal with vector or operator valued random variables. 
We need the following inequalities for Hilbert space valued random variables. 

\begin{lemma} \label{lem:varbd}
Let $\xi$ be a random variable with values in a Hilbert space $\calH$. Then 
for any $h\in\calH$ we have 
$$\bE\left[\left\|\xi-\bE[\xi]\right\|^2\right] \le \bE\left[\|\xi-h\|^2\right].$$
\end{lemma}

\begin{proof}
The proof is quite direct: 
\begin{align*}
\bE\left[\| \xi-\bE[\xi]\|^2\right] & = \bE\left[\left\|(\xi-h)-(\bE[\xi]-h)\right\|^2\right] \\
&= \bE\left[\|\xi-h\|^2\right] -2 \bE\left[\langle \xi-h,\ \bE[\xi]-h\rangle\right] + \|\bE[\xi]-h\|^2 \\
& = \bE\left[\|\xi-h\|^2\right] -\|\bE[\xi]-h\|^2  \\
& \le \bE \left[\|\xi-h\|^2\right].
\end{align*}
\end{proof}

\begin{lemma}\label{lem:var}
Let $\calH$ be a Hilbert space and $\xi$ be a random variable with values in $\calH$.
Assume that $\|\xi\|<M$ almost surely. Let $\{\xi_1,\, \xi_2, \ldots, \xi_n\}$ be
a sample of $n$ independent observations for $\xi.$
Then $$\bE\left[\left\|\frac 1n \sum_{i=1}^n \xi_i -\bE[\xi]\right\|^2\right] \le \frac {M^2}{n}
\qquad\hbox{and} \qquad 
 \bE\left[\left\|\frac 1n \sum_{i=1}^n \xi_i -\bE[\xi]\right\|\right]   \le \frac {M}{\sqrt n}.$$
\end{lemma}

\begin{proof} 
Since $\bE[\xi_i]=\bE[\xi]$ for all $i$ and $\xi_i$ are mutually independent, we have 
\begin{align*}
 \bE\left[\left\|\frac 1n \sum_{i=1}^n \xi_i -\bE[\xi]\right\|^2\right] & = \bE\left[\left \|\frac 1 n \sum_{i=1}^n \xi_i\right\|^2\right] 
 -2\bE\left[\left\langle \frac 1 n \sum_{i=1}^n \xi_i,\, \bE[\xi]\right\rangle\right] + \|\bE[\xi]\|^2 \\
 & = \frac 1{n^2} \sum_{i=1}^n\sum_{j=1}^n \bE\left\langle \xi_i, \, \xi_j\right\rangle - \|\bE[\xi]\|^2 \\
 & = \frac 1 {n} \bE\left[\|\xi\|^2\right] - \frac {1} {n}\| \bE[\xi]\|^2 \\
 & \le \frac 1n \bE\left[\|\xi\|^2\right] \\
 & \le \frac {M^2}{n}.
\end{align*}
This proves the first inequality. The second one follows from the first one and Schwartz inequality.
\end{proof}

In the sequel, we assume $\bx$ is uniformly bounded by $M\ge 0.$ 

\begin{lemma}\label{lem:Sigma}
We have
$$ \bE\left[\|\hat\Sigma-\Sigma\|^2\right] \le \frac{10M^4}{n}  \qquad \hbox{and} \qquad 
\bE\left[\|\hat\Sigma-\Sigma\|\right] \le M^2\sqrt{\frac{10}{n}}$$
\end{lemma}

\begin{proof} 
Let $\mu=\bE[\bx]$ be the mean of $\bx.$ Note that 
$\Sigma =\bE\left[(\bx-\mu)(\bx-\mu)^\top\right]= \bE[\bx\bx^\top]-\mu\mu^\top$ and similarly 
$$\hat\Sigma= \frac 1{n}\sum_{i=1}^n(\bx_i-\bar \bx)( \bx_i-\bar \bx)^\top
= \frac 1{n} \sum_{i=1}^n \bx_i \bx_i^\top -\bar\bx\bar \bx^\top.$$ 
Thus, 
\begin{align} \left\|\hat\Sigma-\Sigma\right\|^2
& \le  2 \left\|\dfrac 1{n} \dsum_{i=1}^n\bx_i\bx_i^\top -\bE\left[\bx\bx^\top\right]\right\|^2  
+ 2 \left\|\bar \bx \bar \bx^\top - \mu\mu^\top\right\|^2 \nonumber\\
& \le  2 \left\|\dfrac 1{n} \dsum_{i=1}^n\bx_i\bx_i^\top -\bE\left[\bx\bx^\top\right]\right\|_F^2 + 8M^2 \|\bar \bx-\mu\|^2, \label{eq:1}
\end{align}
where  $\|\cdot\|_F$ represent the Frobenius norm and we have used the fact that $\|\cdot\|\le \|\cdot\|_F$ for all matrices. 

Recall that matrices of $d\times d$ form a Hilbert space with  the Frobenius norm $\|\cdot\|_F.$ 
Applying Lemma \ref{lem:var} to $\xi=\bx\bx^\top$ which satisfies 
$\|\xi\|_F=\|\bx\|^2\le M^2$, we obtain 
\begin{equation} \bE\left[\left\|\dfrac 1{n} \dsum_{i=1}^n\bx_i\bx_i^\top -\bE\left[\bx\bx^\top\right]\right\|_F^2\right] \le \frac{M^4}{n}.
\label{eq:2} \end{equation}
Next we apply Lemma \ref{lem:var} to $\xi=\bx$  and obtain 
\begin{equation}\label{eq:3} \bE\left[\|\bar\bx-\mu\|^2\right]\le \frac {M^2}{n}.\end{equation}
Then the first inequality follows by taking expectation on both sides of \eqref{eq:1} and applying \eqref{eq:2} and \eqref{eq:3}.
The second one follows from the first one and Schwartz inequality.
\end{proof}

Now we can prove the two theorems.

\begin{proof}[\bf Proof of Theorem \ref{thm:ridge}]
Note that $\bE[\tilde Y|\tilde X] = \tilde X \bw$ and thus  $\bE[\hat\bw]=(\lambda I + \hat\Sigma)^{-1} \hat\Sigma \bw.$
We have 
\begin{align}
& \left\|(\lambda I + \hat\Sigma)^{-1} \hat\Sigma \bw  - (\lambda I + \Sigma)^{-1} \Sigma \bw \right\|  \nonumber \\ 
\le & \left\|(\lambda I + \hat\Sigma)^{-1}(\hat\Sigma -\Sigma)\bw \right\| + \left\|\left((\lambda I + \hat\Sigma)^{-1}-(\lambda I + \Sigma)^{-1}\right) \Sigma \bw\right\| \nonumber \\
 \le &\frac 1\lambda \|\hat\Sigma-\Sigma\| + \left\| (\lambda I + \hat\Sigma)^{-1}(\Sigma-\hat\Sigma)(\lambda I +\Sigma)^{-1} \Sigma \bw\right\| \nonumber \\
 \le & \frac 2 \lambda \|\hat\Sigma-\Sigma\| \|\bw\|.  \nonumber
\end{align}

The conclusion (i) follows from 
\begin{align*} \|\bE[\hat\bw]-(\lambda I +\Sigma)^{-1}\Sigma\bw\| & \le 
 \bE \left[ \left\| (\lambda I + \hat\Sigma)^{-1} \hat\Sigma \bw - (\lambda I + \Sigma)^{-1} \Sigma \bw \right\| \right] \\
& \le \frac {2\|\bw\|}{\lambda} \bE\left[\|\hat\Sigma-\Sigma\|\right] \\ 
& = \frac{2M^2\|\bw\|}{\lambda} \sqrt{\frac{10}{n}} \longrightarrow 0.
\end{align*}

The conclusion (ii) is an easy consequence of (i) by noting that
\begin{align*}
\lim\limits_{n\to\infty} \left\|\bE[\hat\bw] -\bw\right\|^2 & =  \left\|(\lambda I +\Sigma)^{-1} \Sigma \bw-\bw\right\|^2.
\end{align*}

Denote $R=[\epsilon_1, \ldots, \epsilon_n]^\top$ to be the column vector of residuals. 
By the fact $y_i-\bar y = (\bx_i-\bar \bx)^\top \bw + \epsilon_i$ we have 
$\tilde X\tilde Y = \tilde X \tilde X^\top \bw + \tilde XR$
and $$\hat\bw = (\lambda I +\hat\Sigma)^{-1}\hat\Sigma \bw +  (\lambda I +\hat\Sigma)^{-1}\left(\tfrac 1n  \tilde X R\right).$$
By the fact $\bE[\epsilon_i\epsilon_j|\tilde X]=0$ for $i\not=j$ 
and $\|x_i-\bar x\| \le 2M$, we obtain 
\begin{align}
\bE\left[\left\|\tfrac 1n \tilde X R\right\|^2\right] & = \frac 1 {n^2} \sum_{i=1}^n \sum_{j=1}^n
\bE\left[\epsilon_i\epsilon_j(\bx_i-\bar \bx)^\top (\bx_i-\bar \bx)\right]  \nonumber \\
& = \frac 1{n^2} \sum_{i=1}^n \bE\left[\epsilon_i^2\|\bx_i-\bar \bx\|^2\right] \nonumber \\
 & \le \frac {4M^2\bE[\epsilon^2]}{n}. \label{eq:xr}
\end{align}
By Lemma \ref{lem:varbd} and Lemma \ref{lem:Sigma},  we have
\begin{align*}
\bE\left[\left\|\hat\bw -\bE[\hat\bw]\right\|^2\right] & \le \bE\left[\left\|\hat\bw -(\lambda I + \Sigma)^{-1}\Sigma \bw \right\|^2 \right] \\
& \le 2 \bE \left[ \left\| (\lambda I +\hat\Sigma)^{-1}\hat\Sigma \bw - (\lambda I + \Sigma)^{-1}\Sigma \bw \right\|^2\right] 
 + 2 \bE\left[\left\|(\lambda I +\hat\Sigma)^{-1}\left(\tfrac 1n\tilde XR\right)\right\|^2\right] \\
& \le \frac {4\|\bw\|^2} {\lambda^2} \bE\left[ \left \|\hat\Sigma -\Sigma\right\|^2 \right] + 
\frac{2}{\lambda^2} \bE\left[\left\|\tfrac 1 n \tilde XR\right\|^2\right]  \\
& \le \frac{{40M^4\|\bw\|^2}+8M^2\bE[\epsilon^2]}{n \lambda^2}.
\end{align*}
This verifies (iii). 
\end{proof}

\bigskip

\begin{proof}[\bf Proof of Theorem \ref{thm:barr}]
It is easy to verify that 
$ \hat\bw^\sharp =  (\lambda I + \hat\Sigma)^{-2}(2\lambda I + \hat\Sigma) \left(\tfrac 1n \tilde X \tilde Y\right). $
Therefore, $\bE[\hat \bw^\sharp] = \bE\left[(\lambda I + \hat\Sigma)^{-2}(2\lambda I + \hat \Sigma) \hat\Sigma \bw\right].$
To prove (i) we write 
\begin{align*}
& (\lambda I + \hat\Sigma)^{-2}(2\lambda I + \hat \Sigma) \hat\Sigma \bw - (\lambda I + \Sigma)^{-2}(2\lambda I + \Sigma) \Sigma \bw \\
= \ & (\lambda I + \hat\Sigma)^{-2}\left[ (\lambda I +\hat\Sigma)^2 - \lambda^2  I\right] \bw  
- (\lambda I + \Sigma)^{-2}\left[ (\lambda I +\Sigma)^2 - \lambda^2  I\right] \bw\\
=\  & \lambda ^2 \left[(\lambda I + \hat\Sigma )^{-2} - (\lambda I + \Sigma )^{-2} \right] \bw  \\ \
=\ &  \lambda ^2 (\lambda I + \hat\Sigma )^{-2}  \left[ 2\lambda (\Sigma -\hat\Sigma) + (\Sigma -\hat\Sigma)\Sigma + \hat\Sigma(\Sigma-\hat\Sigma)\right] 
(\lambda I + \Sigma)^{-2}\bw.
\end{align*}
By $(\lambda I +\hat\Sigma)^{-2} \le \frac 1 {\lambda^2},$ $(\lambda I +\Sigma)^{-2} \le \frac 1 {\lambda^2},$  
$(\lambda I +\hat\Sigma)^{-2}\hat\Sigma \le \frac 1 {\lambda},$ and 
$\Sigma(\lambda I +\Sigma)^{-2} \le \frac 1 {\lambda},$ we obtain  
\begin{equation}\label{eq:wadiff}
\left\| (\lambda I + \hat\Sigma)^{-2}(2\lambda I + \hat \Sigma) \hat\Sigma \bw - (\lambda I + \Sigma)^{-2}(2\lambda I + \Sigma) \Sigma \bw \right\| 
\le \frac {4\|\bw\|}{\lambda} \left\|\hat\Sigma-\Sigma\right\|.
\end{equation}
The conclusion (i) follows from the following estimate:
\begin{align*}
\|\bE[\hat\bw^\sharp]-\hat\Sigma)^{-2}(2\lambda I + \hat \Sigma) \hat\Sigma \bw\| &
\le \bE\left[ \left\|(\lambda I + \hat\Sigma)^{-2}(2\lambda I + \hat \Sigma) \hat\Sigma \bw - (\lambda I + \Sigma)^{-2}(2\lambda I + \Sigma) \Sigma \bw \right\|\right] \\
& \le \frac{4\|\bw\|}{\lambda} \bE\left[\left\|\hat\Sigma-\Sigma\right\|\right] \le \frac{4M^2\|\bw\|}{\lambda}\sqrt{\frac {10} n}\longrightarrow 0.
\end{align*}

The conclusion (ii) is an easy consequence of (i).

To prove (iii), we write 
$\hat\bw^\sharp = (\lambda I + \hat\Sigma)^{-2}(2\lambda I + \hat \Sigma) \hat\Sigma \bw  
+ (\lambda I + \hat\Sigma)^{-2}(2\lambda I + \hat \Sigma) (\tfrac 1n {\tilde X} R).$
By Lemma \ref{lem:varbd}, Lemma \ref{lem:Sigma}, the estimates in \eqref{eq:wadiff} and \eqref{eq:xr}, we have 
\begin{align*}
\bE\left[\left\|\hat\bw^\sharp -\bE[\hat\bw^\sharp]\right\|^2\right] & \le \bE\left[\left\|\hat\bw^\sharp -(\lambda I + \Sigma)^{-2}(2\lambda I +\Sigma)\Sigma \bw \right\|^2 \right] \\
& \le 2 \bE \left[ \left\| (\lambda I + \hat\Sigma)^{-2}(2\lambda I + \hat \Sigma) \hat\Sigma \bw 
     - (\lambda I + \Sigma)^{-2}(2\lambda I + \Sigma) \Sigma \bw  \right\|^2\right]  \\
& \qquad + 2 \bE\left[\left\|(\lambda I +\hat\Sigma)^{-2}(2\lambda I + \hat\Sigma)  \left(\tfrac 1n\tilde XR\right)\right\|^2\right] \\
& \le \frac {16\|\bw\|^2} {\lambda^2} \bE\left[ \left \|\hat\Sigma -\Sigma\right\|^2 \right] + 
\frac{8}{\lambda^2} \bE\left[\left\|\tfrac 1 n \tilde XR\right\|^2\right]  \\
& \le \frac{{160M^4\|\bw\|^2}+32M^2\bE[\epsilon^2]}{n \lambda^2},
\end{align*}
where we have used the $\|(\lambda I +\hat\Sigma)^{-2}(2\lambda I + \hat\Sigma)\| 
= \|\lambda (\lambda I + \hat\Sigma)^{-2} + (\lambda I + \hat\Sigma)^{-1}\| \le \frac 2\lambda.$ 
\end{proof}

 \section{Bias correction for regularization kernel network}
 \label{sec:barkn}

 When the true regression model is nonlinear, kernel method can be used.
 Denote by $\calX$ the space of explanatory variables. A Mercer kernel is
 a continuous, symmetric, and positive-semidefinite function $K: \calX \times \cal X\to \RR$.
  The inner product defined by $\langle K(\bx, \cdot), K(\bt, \cdot)\rangle_K =K(\bx, \bt)$
 induces a reproducing kernel Hilbert space (RKHS) $\H_K$ associated to the kernel $K$.
 The space is the closure of the function class spanned by $\{K(\bx, \cdot), \bx\in\calX\}.$
 The reproducing property $f(\bx) = \langle f, K(\bx, \cdot)\rangle_K$ leads to 
 $|f(\bx)|\le \sqrt{K(\bx,\bx)} \|f\|_K$. Thus $\calH_K$ can be embedded into $L^\infty.$ 
We refer to  \cite{Aron} for more other properties of RKHS.

 The regularization kernel network \cite{EPP}
 estimates the true model $f^*$ by a function $\hat f \in \H_K$ that
 minimizes the regularized sample mean squared error
 $$\frac 1 n \sum_{i=1}^n (y_i-f(\bx_i))^2 +\lambda \|f\|_K^2.$$
The representer  theorem \cite{wahba1990spline} tells that
$\hat f(\bx) = \sum_{i=1}^n c_i K(\bx_i, \bx).$
So although the RKHS $\H_K$ may be infinitely dimensional,
the optimization of the regularization kernel network could be implemented
in an $n$ dimensional space. Actually, let $\bK$ denote the
kernel matrix on $\bx_1, \ldots, \bx_n$ and  $Y=(y_1, \cdots, y_n)^\top$.
The coefficients $\bc = (c_1, \ldots, c_n)^\top$ could be solved by
a linear system $(\lambda n I + \bK)\bc = Y.$

In \cite{smale2007learning} an operator representation for $\hat f$  is proved. Let $S:\H_K\to \RR^n$ be the sampling operator
defined by $Sf=(f(\bx_1), \ldots, f(x_n))^\top$ for $f\in\H_K.$ Its dual operator, $S^*$ is given by
$S^*\bc = \sum_{i=1}^n c_iK(\bx_i, \cdot) \in \H_K $ for $\bc\in\RR^n.$ Then we have
$$\hat f = \frac 1 n \left(\lambda I + \frac 1 n S^*S\right)^{-1}S^* Y.$$
The operator $\frac 1 n S^*S$ is a sample version of the integral operator $L_K$
define by
$$L_Kf(\bx) = \bE_\bt \left[K(\bx, \bt) f(\bt)\right]= \int_{\calX} K(\bx, \bt) f(\bt) \hbox{d}P_\calX(\bt)$$
where $P_\calX$ is the marginal distribution on $\calX.$
Note that $L_K$ defines a bounded operator both on $L^2$ 
(associate to the probability measure $P_\calX$) and $\calH_K.$
Let $\tau_i$ and $\phi_i$ be the eigenvalues and eigenfunctions of $L_K$.
Then $\{\phi_i\}_{i=1}^\infty$ form an orthogonal basis of $L^2$ and
$$L_Kf = \sum_{i=1}^\infty \tau_i\langle f, \phi_i\rangle_{\!_{L^2}} \phi_i, \qquad \forall\ f\in L^2.$$
Also, $\{\psi_i= \sqrt{\tau_i}\phi_i: \tau_i\not=0\}$ form an orthonormal basis of $\calH_K$ 
and, as an operator on $\calH_K$, 
$$L_Kf = \sum_{i:\tau_i\not=0} \tau_i \langle f, \psi_i\rangle_{\!_K} \psi_i, \qquad \forall \ f\in\calH_K.$$
Moreover, $L_K^{1/2}$ maps all functions in $L^2$ onto $\calH_K$ and 
$$\|f\|_{L^2} = \|L_K^{1/2}f\|_K \qquad \forall \ f\in\hbox{span}\left\{\phi_i: \tau_i\not=0\right\}\subset L^2.$$
In particular, this is true for all $f\in\calH_K.$ Note that $\hbox{span}\left\{\phi_i: \tau_i\not=0\right\}$ is the closure of 
$\calH_K$ in $L^2$. Only functions in $\hbox{span}\left\{\phi_i: \tau_i\not=0\right\}$ can be well approximated 
by functions in $\calH_K$. 

Regularization kernel network can be regarded as a nonlinear extension of ridge regression.
If $f^*\in\H_K$ we can measure the difference between $\hat f $ and $f^*$ in $\H_K$
and prove some conclusions that are analogous to those in Theorem \ref{thm:ridge}.
But unfortunately, this is generally not true. To make our result more general, we measure
the difference between $\hat f$ and $f^*$ in $L^2$ sense, which is
equivalent to measure the mean squared forecasting error.
For this purpose, we define
$$b_K(n, \lambda) = \left\|\bE_D[\hat f]-f^*\right\|_{L^2}= \sqrt{\bE_\bx\left[ \left(\bE_D[\hat f(\bx)]-f^*(\bx)\right)^2\right]}$$
and $$v_K(n, \lambda) = \bE_D\left[ \left\|\hat f -\bE_D[\hat f]\right\|_{L_2}^2\right]
=\bE_{D,\bx}\left[\left(\hat f(\bx)-\bE_D[\hat f(\bx)]\right)^2\right],$$
where $\bE_D$ is the expectation with respect to the data and $\bE_\bx$
is the expectation with respect to $\bx$. 

\begin{theorem} \label{thm:rkn}
 If $\sup_{\bx\in\calX}\sqrt{K(\bx,\bx)}<\infty$  and $|y|\le M$ almost surely,  then
\begin{enumerate}
 \item[{\rm (i)}] $\bE_D[\hat f]$ converges to $(\lambda + L_K)^{-1}L_K f^*$
 in $\H_K.$
 \item[{\rm (ii)}] If $f^* = \sum\limits_{i:\tau_i\not=0} \alpha_i \phi_i$, then
 $$\lim\limits_{n\to \infty} b_K(n, \lambda) 
 =\sqrt{\sum_{i:\tau_i\not=0} \frac {\lambda^2  \alpha_i^2 }{(\lambda + \tau_i)^2}}.$$
 \item[{\rm (iii)}] $v_K(n, \lambda) = O\left(\frac 1 {n^{ 3/ 2} \lambda^2}\right)$ 
 if $\lambda=\lambda(n)$ satisfies $n \lambda \to\infty$ and 
 $\sqrt n \lambda\to 0.$ 
\end{enumerate}
\end{theorem}

Theorem \ref{thm:rkn} (ii) characterizes the asymptotic bias for target functions that belong to 
$\hbox{span}\left\{\phi_i: \tau_i\not=0\right\}$ and thus can be well learned by the regularization kernel network. 
If the target function has a component orthogonal to $\hbox{span}\left\{\phi_i: \tau_i\not=0\right\}$,
the orthogonal component is not learnable and its norm should be added to the right hand side.
The variance bound in Theorem \ref{thm:rkn} (iii) is presented with the assumptions $n \lambda \to\infty$ and 
 $\sqrt n \lambda\to 0$,  which, according to the literature (e.g. \cite{smale2007learning, sun2009note}),
 usually guarantee the regularization kernel network to achieve the optimal learning rate. 
 When this is not true, an explicit bound can be found in the proof in Section \ref{sec:pf2}.

Following the same idea as in Section \ref{sec:barr}, we propose to reduce the bias by using an adjusted function
\begin{equation} \label{eq:fa}
\hat f^\sharp(\bx) = \hat f  + \lambda \left(\lambda I + \frac 1 n S^*S\right)^{-1}  \hat f.
\end{equation}
The implementation of this new approach is easy. We can verify that
 $$\hat f^\sharp(\bx) = \sum_{i=1}^n c_i^\sharp K(\bx_i, \bx)$$ with
 $$\bc^\sharp = \bc + \lambda\left(\lambda I + \frac 1 n\bK\right)^{-1} \bc.$$

Let
$$b^\sharp_K(n, \lambda) = \bE_\bx\left[(\bE_D[\hat f^\sharp(\bx)]-f^*(\bx))^2\right]$$
and $$v^\sharp_K(n, \lambda) = \bE_\bx\left[(\hat f^\sharp(\bx)-\bE_D[\hat f^\sharp(\bx)])^2\right].$$

\begin{theorem}
\label{thm:barkn}
 If $\sup_{\bx\in\calX}\sqrt{K(\bx,\bx)}<\infty$  and $|y|\le M$ almost surely, then
\begin{enumerate}
 \item[{\rm (i)}] $\bE_D[\hat f]$ converges to $(\lambda + L_K)^{-2}(2\lambda I+L_K)L_K f^*$ in $\H_K.$
 \item[{\rm (ii)}] If $f^* = \sum_{i:\tau_i\not=0} \alpha_i \phi_i$, then
 $$\lim\limits_{n\to \infty} b^\sharp_K(n, \lambda) 
 =\sqrt{\sum_{i:\tau_i\not=0} \frac {\lambda^4  \alpha_i^2}{(\lambda + \tau_i)^4}}.$$
 \item[{\rm (iii)}] $v^\sharp_K(n, \lambda) = O(\frac 1 {n^{3/2}  \lambda^2}).$
\end{enumerate}
\end{theorem}

\section{Proofs of Theorem \ref{thm:rkn} and Theorem \ref{thm:barkn}}
\label{sec:pf2}
 
The proofs Theorem \ref{thm:rkn} and Theorem \ref{thm:barkn} are more complicated 
than those of Theorem \ref{thm:ridge} and  Theorem \ref{thm:barr} 
because they require techniques  to handle the estimation of integral operators.
Without loss of generality we assume $\kappa\ge 1,$ $M\ge 1,$ and $\lambda\le 1$
throughout this section in order to simplify our notations.
We will always use $\bE$ for $\bE_D$  in case there is no confusion from the context.

\begin{lemma} \label{lem:qnorm}
Let $\xi$ be positive random variable. For any $q\ge 1$, 
$\bE[\xi^q] = \int_0^\infty q t^{q-1} \Pr[\xi>t] dt.$
\end{lemma}

The following concentration inequality is proved in \cite{smale2007learning}.
\begin{lemma}\label{lem:BennetH2}
Let $\calH$ be a Hilbert space and $\xi$ be a random variable with values in $\calH$.
Assume that $\|\xi\|<M$ almost surely. Let $\{\xi_1,\, \xi_2, \ldots, \xi_m\}$ be
a sample of $m$ independent observations for $\xi.$
Then for any $0<\delta<1,$ 
\begin{equation}\label{eq:BernBd}
\left\|\frac 1m\sum_{i=1}^m \xi_i -\bE(\xi) \right\| \le \frac{2M \log(2/\delta)}{n} + \sqrt{\frac{2\bE[\|\xi\|^2]\log(2/\delta)}{n}}.
\end{equation}
\end{lemma}.

When no information regarding $\bE[\|\xi\|^2]$ is available, 
we can use $\bE[\|\xi\|^2]\le M^2$ to derive a simpler estimation as follows. 
For any $0<\delta<1,$  with confidence at least $1-\delta$
\begin{equation}\label{eq:BennetBd}
\left\|\frac 1m\sum_{i=1}^m \xi_i -\bE(\xi) \right\| \le 2M\sqrt{\frac{\log(2/\delta)}{m}}.
\end{equation}

Before we state the next lemma, let us recall the Hilbert-Schmidt operators on $\calH_K$.
Let $\{e_i\}$  be a set of orthonormal basis of $\calH_K$. 
An operator $T$ is a Hilbert-Schmidt operator if 
$\|T\|_{HS}^2 =\sum_i\|Te_i\|_K^2$ is finite. A Hilbert-Schmidt operator is also 
a bounded operator with $\|T\|\le \|T\|_{HS}.$ 
All Hilbert-Schmidt operators form a Hilbert space. 
 For $g, h\in\calH_K$, the rank one tensor operator $g\otimes h$ is defined by
 $(g\otimes h)f = \langle h, \, f\rangle_K g$ for all $f\in\calH_K.$
 A tensor operator is a Hilbert-Schmidt operator with $\|g\otimes h\|_{HS} =\|g\|_K\|h\|_K.$

\begin{lemma}\label{lem:LK}
For any $0<\delta<1,$ we have 
$$\left\|\frac 1n S^*S-L_K\right\| \le 2\kappa^2\sqrt{\frac{\log(2/\delta)}{m}}$$
with confidence at least $1-\delta.$ 
Also, we have 
$$\bE\left[\left\|\frac 1n S^*S-L_K\right\|^2 \right] \le \frac{\kappa^4}{n} \quad \hbox{and}\quad 
 \bE\left[\left\|\frac 1n S^*S-L_K\right\| \right] \le \frac{\kappa^2}{\sqrt n}.$$
\end{lemma}

\begin{proof}
Consider the random variable $\xi=K(\bx, \cdot) \otimes K(\bx, \cdot).$ It satisfies 
$\|\xi\|_{HS} = \|K(\bx, \cdot)\|_K^2=K(\bx,\bx)\le \kappa^2$,  $\bE[\xi]=L_K,$ 
and $\frac 1nS^*S=\frac 1n \sum_{i=1}^n K(\bx_i, \cdot)\otimes K(\bx_i, \cdot).$ 
Then the first inequality follows from \eqref{eq:BennetBd}.

The second and third inequalities have been obtained in \cite{devito2005model, sun2009application}.
They also follow from Lemma \ref{lem:var} easily.
\end{proof}

\begin{lemma} \label{lem:SSf}
For any $0<\delta<1,$ we have 
$$\left\|\frac 1n \sum_{i=1}^n f(\bx_i)K(\bx_i, \cdot) -L_Kf^*\right\| \le 2\kappa M\sqrt{\frac{\log(2/\delta)}{n}}$$
with confidence at least $1-\delta.$ 
Also, we have 
$$\bE\left[\left\| \frac 1n \sum_{i=1}^n f(\bx_i)K(\bx_i, \cdot)-L_Kf^*\right\|^2 \right] \le \frac{\kappa^2 M^2}{n} 
\quad \hbox{and}\quad 
\bE\left[\left\| \frac 1n \sum_{i=1}^n f(\bx_i)K(\bx_i, \cdot) -L_Kf^*\right\| \right] \le \frac{\kappa M}{\sqrt n}.$$
\end{lemma}

\begin{proof} Consider the random variable $\xi = f^*(\bx)K(\bx, \cdot).$ 
Clearly $\xi\in\calH_K$ and $\bE[\xi]=L_Kf^*$. 
Since $|y|\le M$ almost surely,  $|f^*(\bx)|\le M$ for all $\bx\in\cal X.$ Thus,
$$\|\xi\|_K =|f(\bx)|\sqrt{K(\bx, \bx)}\le \kappa \|f^*\|_\infty\le \kappa M.$$ 
Then the conclusions 
follow from applying \eqref{eq:BennetBd} and Lemma \ref{lem:var}, respectively.
\end{proof}

\begin{remark}
If we extend the sampling operator $S$ to  be defined on $L^2$,
then the representation $\frac 1n S^*S f^*$ makes sense and 
$$\frac 1n S^*S f^* = \frac 1n \sum_{i=1}^n f(\bx_i)K(\bx_i, \cdot). $$
In the sequel we will always adopt this simplified notation. However, we need to keep in mind 
that,  in a general situation where $f^*\notin \calH_K,$ 
$\frac 1n S^*S$ should not be regarded as an operator on $\calH_K$.
Therefore, we cannot get Lemma \ref{lem:SSf} from  Lemma \ref{lem:LK} directly.
\end{remark}

Denote $f_\lambda = (\lambda I + L_K)^{-1}L_K f^*$. 
It is known from \cite{smale2007learning} that 
$$f_\lambda = \arg\min_{f\in\calH_K}  \left\{\bE [(y-f(\bx))^2]+\lambda \|f\|_K^2 \right\}.$$
So we have
\begin{equation}\label{eq:Eflbd}
\bE[(y-f_\lambda(\bx))^2] + \lambda\|f_\lambda\|_K^2 \le \bE[(y-{\bf 0})^2]+\lambda\|\mathbf 0\|_K^2 \le M^2
\end{equation}
and as a result \begin{equation} \label{eq:flbd} 
\|f_\lambda\|_K\le\frac{M}{\sqrt \lambda}.
\end{equation}
We will need the following lemma to in our proofs.
\begin{lemma}\label{lem:U}
For any $0<\delta<1$, we have 
$$\left\| \frac 1 n S^*(Y-Sf_\lambda)-L_K(f^*-f_\lambda)\right\|_K 
\le \kappa M \left(\frac{4}{\sqrt n} + \frac{2\kappa }{n\sqrt \lambda}\right) \log(2/\delta)$$
with confidence at least $1-\delta$.
\end{lemma}

\begin{proof} Consider the random variable $\xi= (y-f_\lambda(\bx))K(\bx, \cdot).$ Then by \eqref{eq:flbd}
$$\|\xi\|_K = |y-f_\lambda(\bx)|\sqrt{K(\bx, \bx)}\le \kappa (M+\kappa\|f_\lambda\|_K)$$
and  by \eqref{eq:Eflbd}
$$\bE[\|\xi\|_K^2] = \bE\left[(y-f_\lambda (\bx))^2K(\bx,\bx)\right]\le \kappa^2
\bE\left[(y-f_\lambda (\bx))^2\right] \le \kappa^2 M^2.$$
Applying Lemma \ref{lem:BennetH2} we obtain
$$\left\| \frac 1 n S^*(Y-Sf_\lambda)-L_K(f^*-f_\lambda)\right\|_K 
\le \frac{2\kappa(M+\kappa\|f_\lambda\|_K)\log(2/\delta)}{n} + \sqrt{\frac{2\kappa^2M^2\log(2/\delta)}{n}}.$$
The desired bound follows by using \eqref{eq:flbd} and noticing 
$\sqrt{2\log(2/\delta)}\le \sqrt{\frac 2{\log(2)}} \log(2/\delta) \le 2 \log(2/\delta).$
\end{proof}

\bigskip
Now we can prove Theorem \ref{thm:rkn}.

\begin{proof}[\bf Proof of Theorem \ref{thm:rkn}.]
Note that $\bE[\hat f]  = \bE \left[(\lambda I + \tfrac 1 n S^*S)^{-1}(\tfrac 1n S^*Sf^*)\right].$  
By Lemma \ref{lem:LK} and Lemma \ref{lem:SSf}, we have can prove (i) as follows:
\begin{align}
 \left\|\bE [\hat f] - f_\lambda \right\|_K \nonumber
\le\ & 
\bE \left[ \left\| (\lambda I + \tfrac 1 n S^*S)^{-1}(\tfrac 1n S^*Sf^*)  - f_\lambda\right\|_K \right]\nonumber \\
 \le\ &  \bE \left[ \left\| (\lambda I + \tfrac 1 n S^*S)^{-1}(\tfrac 1n S^*Sf^*-L_Kf^*) \right\|_K \right] \nonumber  \\
& \quad  + \bE \left[\left\|  \left\{(\lambda I + \tfrac 1n S^*S)^{-1}- (\lambda I + L_K)^{-1}\right\} L_K f^*\right\|_K \right] \nonumber \\
\le\ & \frac 1\lambda \bE\left[\left\|\tfrac 1n S^*Sf^*-L_Kf^*\right\|_K\right]  \nonumber \\
& \quad  + \bE \left[\left\| 
(\lambda I + \tfrac 1n S^*S)^{-1}(L_K-\tfrac 1n S^*S) (\lambda I + L_K)^{-1} L_K f^*\right\|_K \right] \nonumber\\
 \le\ & \frac {\kappa M}{\lambda\sqrt n}  + 
 \frac {\left\|f_\lambda\right\|_K}{\lambda}  \bE\left[\left\|L_K-\tfrac 1n S^*S\right\|\right] \nonumber\\
 \le\ & \frac {\kappa M+\kappa^2\left\|f_\lambda\right\|_K}{\lambda\sqrt n} \longrightarrow 0. 
 \label{eq:Q1}
\end{align} 

Since the convergence in $\calH_K$ implies convergence in $L_2$, $\bE [\hat f]$ also converges to 
$(\lambda I +L_K)^{-1} L_K f^*$ in $L_2.$ Therefore the conclusion (ii) holds.

To prove (iii), note that it is verified in \cite{smale2007learning} that 
 $\lambda f_\lambda = L_K(f^*-f_\lambda)$  and 
 \begin{align*} 
 \hat f -f_\lambda & =\left(\lambda I + \tfrac 1n S*S\right)^{-1}\Big\{\tfrac 1 n S^*(Y-Sf_\lambda) -L_K(f^*-f_\lambda)\Big\}.
  \end{align*} 
  By the fact that $\|T\|^2 = \|T^*T\|$ for an operator $T$, we have 
  \begin{align} \left\|L_K^{1/2}\left(\lambda I + \tfrac 1n S*S\right)^{-1}\right\|^2 
  = \ & \left\| \left(\lambda I + \tfrac 1n S*S\right)^{-1} L_K \left(\lambda I + \tfrac 1n S*S\right)^{-1}\right\| \nonumber \\
  \le \ & \left\| \left(\lambda I + \tfrac 1n S*S\right)^{-1} (L_K - \tfrac 1n S^*S) 
  \left(\lambda I + \tfrac 1n S*S\right)^{-1}\right\| \nonumber \\
& + \left\| \left(\lambda I + \tfrac 1n S*S\right)^{-1} \left( \tfrac 1n S^*S\right)  
\left(\lambda I + \tfrac 1n S*S\right)^{-1}\right\|  \nonumber\\
  \le \ & \frac 1 \lambda\left\|L_K-\tfrac 1 n S^*S\right\| + \frac 1{ \lambda}. \label{eq:LKS}
  \end{align}
  By Lemma \ref{lem:LK} and Lemma \ref{lem:U}, we obtain 
  \begin{align*} 
  \left\|\hat f- f_\lambda\right\|_{L^2}^2 & =  \left\|L_K^{1/2}(\hat f-f^*)\right\|_K^2  \\[1mm]
  & \le  \left\|L_K^{1/2}\left(\lambda I + \tfrac 1n S*S\right)^{-1}\right\|^2 
      \left\| \frac 1 n S^*(Y-Sf_\lambda)-L_K(f^*-f_\lambda)\right\|_K ^2 \\[1mm]
  & \le \left(\frac{2 \kappa^2}{\lambda^2}  \sqrt{\frac{\log(2/\delta)} n }  +\frac 1{\lambda} \right)
  \kappa^2M^2 \left(\frac 4 {\sqrt n} +\frac \kappa {n\sqrt \lambda}\right)^2 \Big(\log(2/\delta)\Big)^2\\[1mm]
  & \le \kappa^2 M^2 \left(\frac {2\kappa^2}{\lambda^2 \sqrt n } + \frac 2{\lambda }\right)
  \left(\frac{32}{n}+\frac {2\kappa^2}{n^2\lambda}\right)\Big(\log(2/\delta)\Big)^{5/2}.
  \end{align*}
  with confidence at least $1-2\delta.$
  Denote 
 $$\varsigma(n,\lambda) =\kappa^2 M^2 \left(\frac {2\kappa^2}{\lambda^2 \sqrt n } + \frac 2{\lambda }\right)
  \left(\frac{32}{n}+\frac {2\kappa^2}{n^2\lambda}\right).$$
  The random variable $\xi =  \|\hat f- f_\lambda\|_{L^2}^{4/5}$ is positive and satisfies 
  $\Pr\left[\xi\le \varsigma^{2/5} \log(2/\delta)\right] \ge 1 -2\delta$ or equivalently 
  $$\Pr\left[\xi>t\right] \le 4 \exp\left(-\frac{t}{\varsigma^{2/5}(n, \lambda)}\right).$$
  Applying Lemma \ref{lem:qnorm} to $\xi$ with $q=5/2$ we obtain
  $$\bE_D\left[ \|\hat f -\bE_D[\hat f]\|_{L^2}^2\right] \le \bE\left[\|\hat f- f_\lambda\|_{L^2}^2\right]
  \le  10\int_0^\infty t^{3/2} \exp \left(-\frac{t}{\varsigma^{2/5}(n, \lambda)}\right)dt
  =\tfrac{15}{2}\sqrt\pi \varsigma(n,\lambda).
  $$
  If $\lambda n\to \infty$ and $\lambda\sqrt n \to 0$, by \eqref{eq:flbd}, we can verify that 
  $\varsigma(n, \lambda) = O(\frac{1}{\lambda^2 n^{3/2}}).$ 
  This proves (iii).
  \end{proof}

\bigskip

Denote $f_\lambda^\sharp = (\lambda I + L_K)^{-2}(2\lambda I +L_K)L_K f^*$. We can verify that 
\begin{equation}\label{eq:fladec}
f_\lambda^\sharp = f_\lambda +\lambda (\lambda+L_K)^{-1}f_\lambda.
\end{equation}
This together with \eqref{eq:flbd} implies
\begin{equation}\label{eq:flabd}
\|f_\lambda^\sharp\|_K\le 2\|f_\lambda\|_K\le \frac{2M}{\sqrt \lambda}.
\end{equation}
We need the following two lemmas in the proof of Theorem \ref{thm:barkn}.

\begin{lemma}\label{lem:Ua}
For any $0<\delta<1$, we have 
$$\left\|\frac 1 n S^*(Y-Sf_\lambda^\sharp)-L_K(f^*-f_\lambda^\sharp)\right\|_K 
\le \kappa M \left(\frac {4\kappa} {n\sqrt \lambda} + \frac 5 {\sqrt n}\right) \log(2/\delta)$$
with confidence at least $1-\delta$.
\end{lemma}

\begin{proof} Consider the random variable $\xi= (y-f_\lambda^\sharp(\bx))K(\bx, \cdot).$ By \eqref{eq:flabd}, we have
$$\|\xi\|_K = |y-f_\lambda^\sharp(\bx)|\sqrt{K(\bx, \bx)}\le \kappa (M+2\kappa\|f_\lambda\|_K).$$
By \eqref{eq:fladec} and \eqref{eq:Eflbd},  we have
\begin{align*} \bE\left[\|\xi\|_K^2\right] & = \bE\left[(y-f_\lambda^\sharp (\bx))^2K(\bx,\bx)\right]\le \kappa^2
\bE\left[(y-f_\lambda^\sharp (\bx))^2\right] \\
& \le \kappa^2\left(2 \bE\left[(y-f_\lambda(\bx))^2\right] 
+ 2 \lambda^2 \left\| L_K^{1/2} (\lambda I + L_K)^{-1} f_\lambda\right\|_K^2\right)  \\
& \le \kappa^2\left(2 \bE\left[(y-f_\lambda(\bx))^2\right] 
+ 2 \lambda \left\| f_\lambda\right\|_K^2\right)  \\
& \le 2 \kappa^2 M^2.
\end{align*}
Applying Lemma \ref{lem:BennetH2}, we obtain 
$$\left\|\frac 1 n S^*(Y-Sf_\lambda^\sharp)-L_K(f^*-f_\lambda^\sharp)\right\|_K 
\le \frac{2\kappa(M+2\kappa\|f_\lambda\|_K)\log(2/\delta)}{n} + \sqrt{\frac{4\kappa^2M^2\log(2/\delta)}{n}},$$
which in combination with \eqref{eq:flbd} implies the desired conclusion.
\end{proof} 

\begin{lemma}\label{lem:SSdf}
Let $g=\lambda(\lambda I +L_K)^{-2} L_K f^*.$ For any $0<\delta<1$, we have 
$$\left\|\frac 1 n S^*S g -L_Kg\right\|_K 
\le \kappa M \left(\frac {2\kappa} {n\sqrt \lambda} + \frac 2 {\sqrt n}\right) \log(2/\delta)$$
with confidence at least $1-\delta$.
\end{lemma}

\begin{proof} Consider the random variable $\xi= g(\bx)K(\bx, \cdot).$  
Note that $g = \lambda (\lambda I + L_K)^{-1} f_\lambda.$ 
So $\|g\|_K \le \|f_\lambda\|_K.$ Therefore, 
$$\|\xi\|_K = |g(\bx) |\sqrt{K(\bx, \bx)}\le \kappa^2 \|f_\lambda\|_K.$$
We also have
$$ \bE\left[\|\xi\|_K^2\right]  = \bE\left[(g (\bx))^2K(\bx,\bx)\right]
\le \kappa^2 \|g\|_{L^2}^2 \le \kappa^2 \|f^*\|_{L^2}^2 \le \kappa^2 M^2.$$
Applying Lemma \ref{lem:BennetH2}, we obtain 
$$\left\|\frac 1 n S^*S g -L_Kg\right\|_K 
\le \frac{2\kappa^2\|f_\lambda\|_K\log(2/\delta)}{n} + \sqrt{\frac{2\kappa^2M^2\log(2/\delta)}{n}},$$
which in combination with \eqref{eq:flbd} implies the desired conclusion.
\end{proof} 

\begin{proof}[\bf Proof of Theorem \ref{thm:barkn}.]
Note that 
\begin{align*} \bE[\hat f^\sharp] &  = \bE \left[(\lambda I + \tfrac 1 n S^*S)^{-2}(2\lambda I + \tfrac 1n S^*Sf)(\tfrac 1n S^*Sf^*)\right] \\[1mm]
& = \left(\lambda I + \tfrac 1 n S^*S\right)^{-1} \left(\tfrac 1n S^*Sf^*\right) + \lambda (\lambda I + \tfrac 1n S^*S)^{-2}(\tfrac 1n S^*S f^*)
\end{align*}
and 
 \begin{equation*}
 f_\lambda^\sharp=f_\lambda+ \lambda (\lambda I + L_K)^{-2}L_Kf^*.
\end{equation*}
So we can write
\begin{align*}
   \bE[\hat f^\sharp] - f_\lambda^\sharp 
= \ & \Big\{ (\lambda I + \tfrac 1 n S^*S)^{-1}(\tfrac 1n S^*Sf^*) -f_\lambda \Big\} \\[1mm]
& +  \lambda \left(\lambda I + \tfrac 1n S^*S\right)^{-2} \left(\tfrac 1n S^*S f^* - L_Kf^*\right) \\[1mm]
& + \lambda\Big\{(\lambda I + \tfrac 1n S^*S)^{-2} - (\lambda I +L_K)^{-2}\Big\} L_Kf^* \\
  =\ &  Q_1 + Q_2+Q_3.
 \end{align*} 
 By \eqref{eq:Q1} we have $\bE[\|Q_1\|_K] \longrightarrow 0.$ For $Q_2$, by Lemma \ref{lem:SSf}, we have 
 $$\bE[\|Q_2\|_K] \le \frac 1 \lambda \bE\left[\left\|\tfrac 1n S^*S f^* -L_Kf^*\right\|_K\right] \le \frac{\kappa M}{\lambda \sqrt n } 
 \longrightarrow 0.$$
 For $Q_3,$ we can verify that 
\begin{equation*} Q_3  = 
\lambda(\lambda I + \tfrac 1n S^*S)^{-2} \Big[2\lambda \left(L_K-\tfrac 1n S^*S\right) 
 +\left(L_K-\tfrac 1n S^*S\right)L_K 
 + \tfrac 1n S^*S  \left(L_K-\tfrac 1n S^*S\right) \Big] (\lambda I +L_K)^{-1}f_\lambda.
 \end{equation*}
So $\|Q_3 \|_K \le \frac {4 \|f_\lambda\|} \lambda  \left\| L_K-\tfrac 1n S^*S \right\|.$
By Lemma \ref{lem:LK}, we obtain 
$$\bE\left[\|Q_3\|_K\right] \le \frac{4\kappa^2 \|f_\lambda\|_K}{\lambda\sqrt{n}} \longrightarrow 0.$$
Combining the estimation for $Q_1$, $Q_2$ and $Q_3$, we obtain
$$
 \left\|\bE [\hat f^\sharp]- (\lambda I + L_K)^{-2}(2\lambda I + L_K)L_Kf^*\right\|_K \\
\le \bE [\|Q\|_K ] \longrightarrow 0.$$
This proves (i).

The convergence in $\calH_K$ implies convergence in $L_2$ and the latter implies the conclusion (ii).

To prove (iii), 
we first observe that
$$\lambda^2 f_\lambda^\sharp = 2\lambda L_K(f^* -f_\lambda) + L_K^2(f^*-f_\lambda^\sharp).$$
So we can write 
\begin{align*}
\hat f^\sharp - f_\lambda^\sharp & = (\lambda I +\tfrac 1n S^*S)^{-2} \left[ \left(2\lambda I +  \tfrac 1n S^*S\right) \tfrac 1 n S^*Y 
 - \left(\lambda I + \tfrac 1n S^*S\right)^2 f_\lambda^\sharp\right ]\\
 & = 2\lambda (\lambda I +\tfrac 1n S^*S)^{-2} \left[ \tfrac 1 n S^*(Y -Sf_\lambda^\sharp) - L_K(f^*-f_\lambda^\sharp)\right] \\
 & \quad +  (\lambda I +\tfrac 1n S^*S)^{-2}\left[\left(\tfrac 1n S^*S\right) \left(\tfrac 1n S^*(Y-Sf_\lambda^\sharp) \right) 
  - L_K^2(f^*-f_\lambda^\sharp)\right] \\
  & = 2\lambda (\lambda I +\tfrac 1n S^*S)^{-2} \left[ \tfrac 1 n S^*(Y -Sf_\lambda^\sharp) - L_K(f^*-f_\lambda^\sharp)\right] \\
 & \quad +  (\lambda I +\tfrac 1n S^*S)^{-2}\left(\tfrac 1n S^*S\right) \left[\tfrac 1n S^*(Y-Sf_\lambda^\sharp) - L_K(f^*-f_\lambda^\sharp) \right]\\
& \quad +  (\lambda I +\tfrac 1n S^*S)^{-2}\left(\tfrac 1n S^*S-L_K\right) L_K(f^*-f_\lambda^\sharp)\\
&=\tilde Q_1  + \tilde Q_2 + \tilde Q_3.
\end{align*}
 It is easy to check that \begin{align*}
  \|L_K^{1/2} \tilde Q_1\|_K  \le 2 \left\|L_K^{1/2}(\lambda I +\tfrac 1n S^*S)^{-1}\right\|  
  \Big \|\tfrac 1n S^*(Y-Sf_\lambda^\sharp)-L_K(f^*-f_\lambda^\sharp)\Big\|_K 
  \end{align*}
and
$$ \|L_K^{1/2} \tilde Q_2\|_K  \le \left\|L_K^{1/2}(\lambda I +\tfrac 1n S^*S)^{-1}\right\|  
  \Big \|\tfrac 1n S^*(Y-Sf_\lambda^\sharp)-L_K(f^*-f_\lambda^\sharp)\Big \|_K. $$
For $\tilde Q_3$, we observe that 
$L_K(f^*-f_\lambda^\sharp) = \lambda^2(\lambda I +L_K)^{-2}L_Kf^*.$
Denote $g=\lambda(\lambda I +L_K)^{-2} L_K f^*$.  We have
$$\|L_K^{1/2} \tilde Q_3\|_K \le 
\left\|L_K^{1/2}(\lambda I +\tfrac 1n S^*S)^{-1}\right\|  \Big\|\tfrac 1n S^*Sg - L_Kg\Big\|_K .$$
By \eqref{eq:LKS}, Lemma \ref{lem:LK},  Lemma \ref{lem:Ua} and Lemma \ref{lem:SSdf}, we obtain  
 \begin{align*}
&  \left\|\hat f^\sharp - f_\lambda^\sharp\right\|_{L^2}^2 =  \left\|L_K^{1/2} (\hat f^\sharp - f_\lambda^\sharp)\right\|_K   \\[1mm]
\le & \left\|L_K^{1/2}(\lambda I +\tfrac 1n S^*S)^{-1}\right\| ^2 \bigg(
2 \Big \|\tfrac 1n S^*(Y-Sf_\lambda^\sharp)-L_K(f^*-f_\lambda^\sharp)\Big \|_K + \Big\|\tfrac 1n S^*Sg - L_Kg\Big\|_K\bigg)^2 \\[1mm]
  \le\  &  \left(\frac{2 \kappa^2}{\lambda^2}  \sqrt{\frac{\log(2/\delta)} n }  +\frac 1{\lambda} \right)
 \kappa^2  M^2 \left( \frac {10\kappa} {n\sqrt \lambda} + \frac {12} {\sqrt n}\right)^2 \Big(\log(2 /\delta)\Big)^2 \\[1mm]
 \le \  &  \kappa^2 M^2  \left(\frac{2 \kappa^2}{\lambda^2\sqrt n}    +\frac 2{\lambda} \right)
\left( \frac {10\kappa} {n\sqrt \lambda} + \frac {12} {\sqrt n}\right)^2 \Big(\log(2 /\delta)\Big)^{5/2}
\end{align*}
  with probability at least $1-3\delta.$
  This implies 
  $$\bE\left[ \|\hat f^\sharp -\bE[\hat f^\sharp]\|_{L^2}^2\right] \le \bE\left[\|\hat f^\sharp- f_\lambda^\sharp\|_{L^2}^2\right]
   ={\large\tfrac{45}{4}}\sqrt\pi \kappa^2 M^2  \left(\frac{2 \kappa^2}{\lambda^2\sqrt n}    +\frac 2{\lambda} \right)
\left( \frac {10\kappa} {n\sqrt \lambda} + \frac {12} {\sqrt n}\right)^2. $$
It is of order $O(\frac{1}{\lambda^2 n^{3/2}})$ 
  when  $\lambda n\to \infty$ and $\lambda\sqrt n \to 0.$ 
  This proves (iii).
\end{proof}

 \section{Learning with block wise data}
 \label{sec:block}

 In learning with block wise  streaming data,
 let $D_t = \{(\bx_{it}, y_{it})\}_{i=1}^n$ be the data block
 collected at time $t.$ Consider the simple incremental algorithm
 which learns a function $\hat f_t$ from $D_t$ by a base algorithm. Then the target function
 for prediction uses the average of all the learnt functions available upon time $t,$
 that is,  $$\bar f_t (\bx)= \frac 1 t \sum_{i=1}^t \hat f_i(\bx).$$
 We see the target function is updated incrementally:
 $$\bar f_{t} = \frac {t-1} t \bar f_{t-1} + \frac 1 t \hat f_t.$$

Let $b$ and $v$ denotes
 the bias and variance of $\hat f_i$. The mean squared error of $\bar f_t$
 is $$\hbox{mse}(\bar f_t) = b^2 + \frac {v}{t}.$$ When $t$ becomes large, we
 see the variance term shrinks but the bias term does not.
 So the performance of the simple incremental learning method is dominated by the bias.
 Hence, a base algorithm with small bias is preferred.
 This intuition supports the use of bias corrected algorithms.

 In divide and conquer algorithm, data blocks do not arrive in time.
 Instead, they are artificially generated. But from a computational perspective,
 its idea is the same as the simple incremental learning.
 So bias correction is also preferred.

 \section{Simulation study}
 \label{sec:sim}

 In this section, we illustrate the use of bias corrected algorithms by a variety of simulations.
 We will first study their performance on a single data set. Then
 we verify their effectiveness in learning with block wise data
using both simulated data and real world data.

\subsection{Learning with a single data set}

Let $\bx\in\RR^{20}$ and all the 20 explanatory variables
are independent and normally distributed. Let the $i$th variable $x_i$ has variance $2^{-i}$. So
the $i$-th variable is also the $i$th principal component.
We consider two linear models where
$$ \begin{array}{lll}
\hbox{ Model 1:} & \bw_1=[1, 1, -1, -1, 0, \ldots, 0], & b_1=0;\\
\hbox{ Model 2:} & \bw_2=[0, \ldots, 0, 1, 1, -1, -1], & b_2=0.
\end{array}$$
Note that the first model depends on the first four principal components and
the second one depends on the last four principal components.
For both models, the noise level is set such that
the signal to noise ratio is 10.
We use the sample size $n=100.$

We first compare the bias, variance, and mean squared error of ridge regression (RR)
and bias corrected ridge regression (BCRR).
They are calculated by averaging the bias,
variance, and mean squared error
after repeating the experiment 1000 times.
The results are shown in Figure \ref{fig:bias1} and Figure \ref{fig:bias2}. They show that,
within a reasonable range of $\lambda$, the BCRR has smaller bias and larger variance.
Their minimum mean squared errors are comparable. BCRR can achieve the minimum mean squared error
with a larger regularization parameter. Since the linear system could be more stable with larger $\lambda$,
BCRR may be superior in very high dimensional situation where the covariance matrix is near singular.

\begin{figure}[htbp]
\centerline{\includegraphics[width=0.32\textwidth]{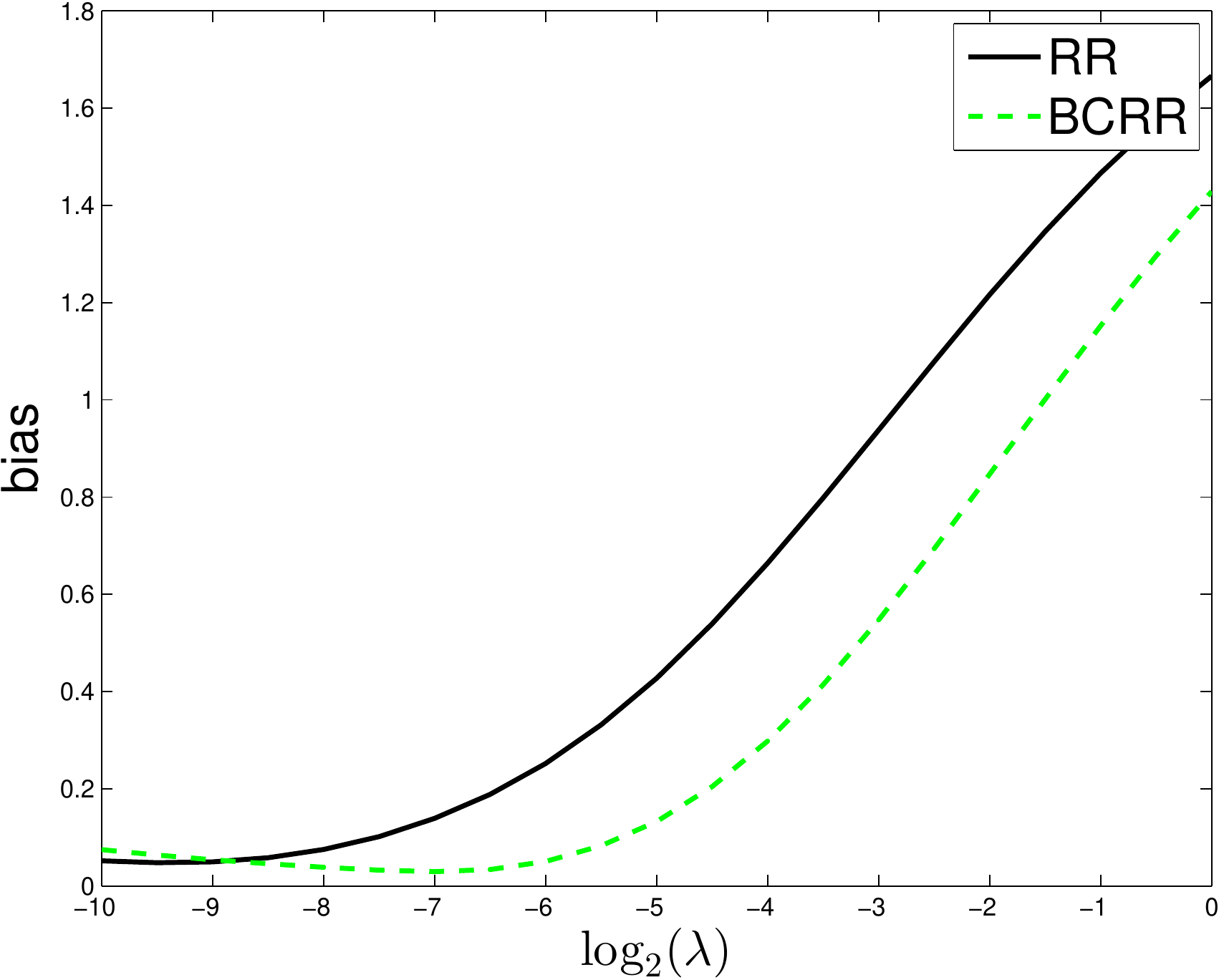}
\quad \includegraphics[width=0.32\textwidth]{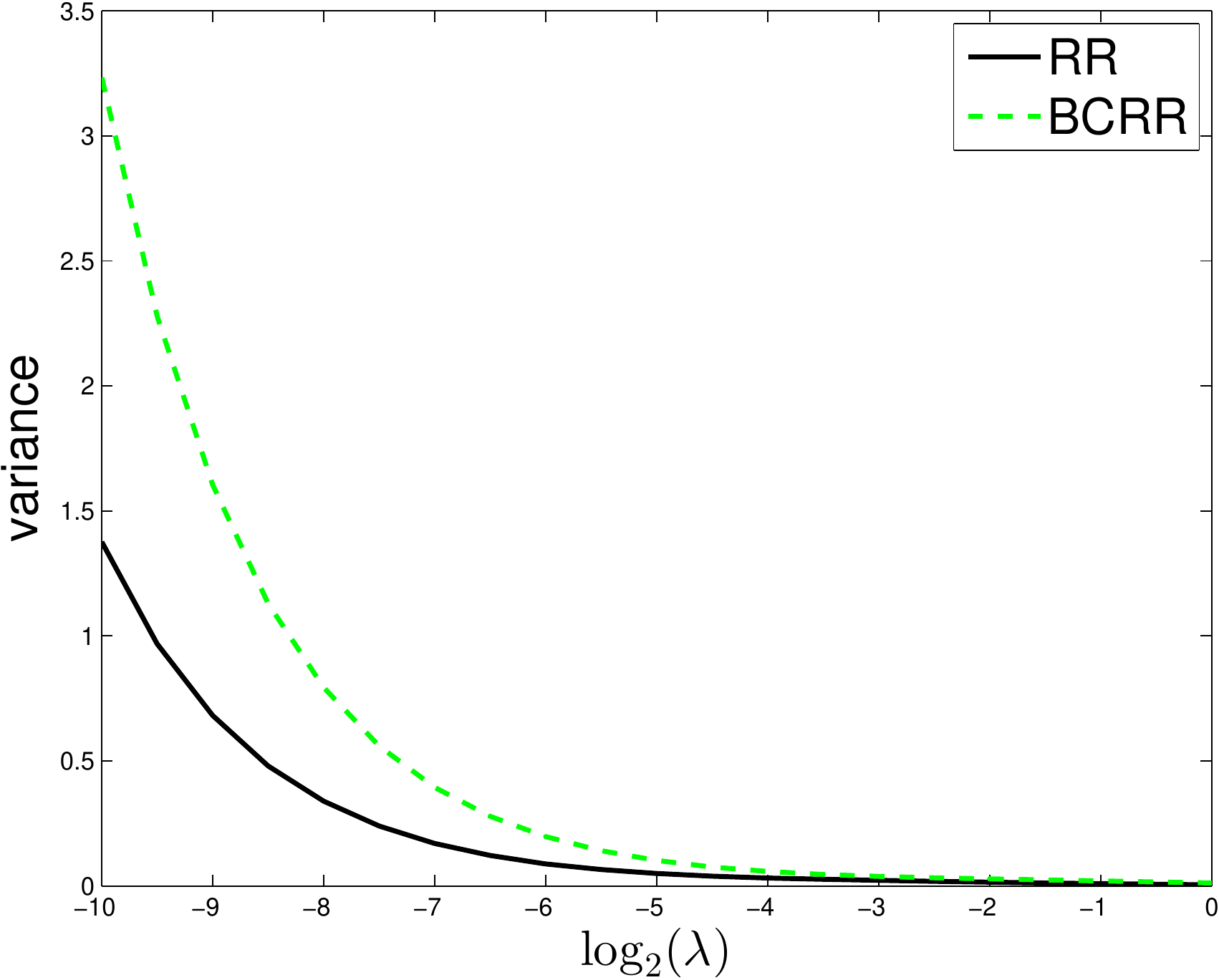}
\quad \includegraphics[width=0.32\textwidth]{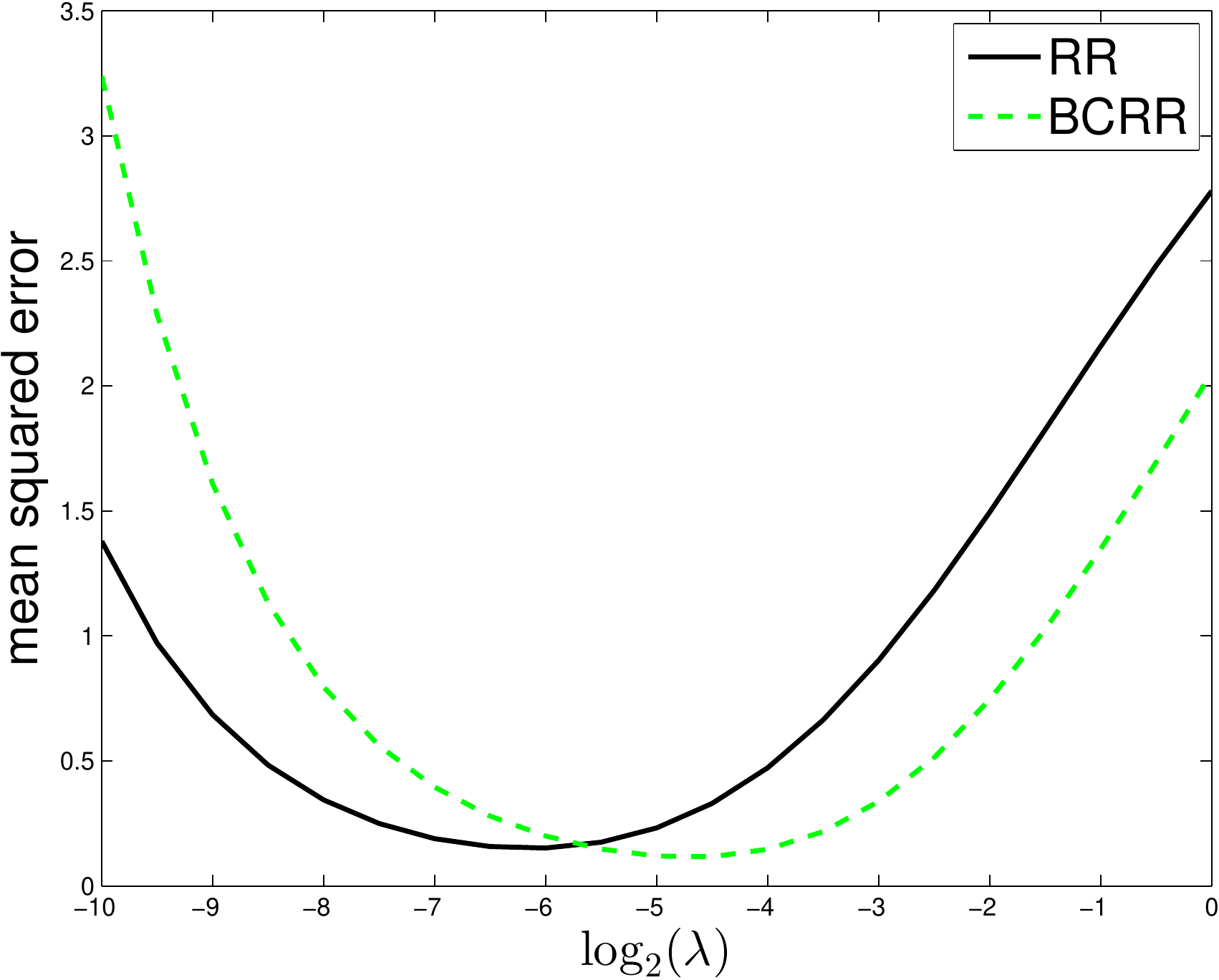}}
\caption{Model 1: Bias, variance, and mean squared error of
RR and BCRR.
\label{fig:bias1}}
\end{figure}

\begin{figure}[htbp]
\centerline{\includegraphics[width=0.32\textwidth]{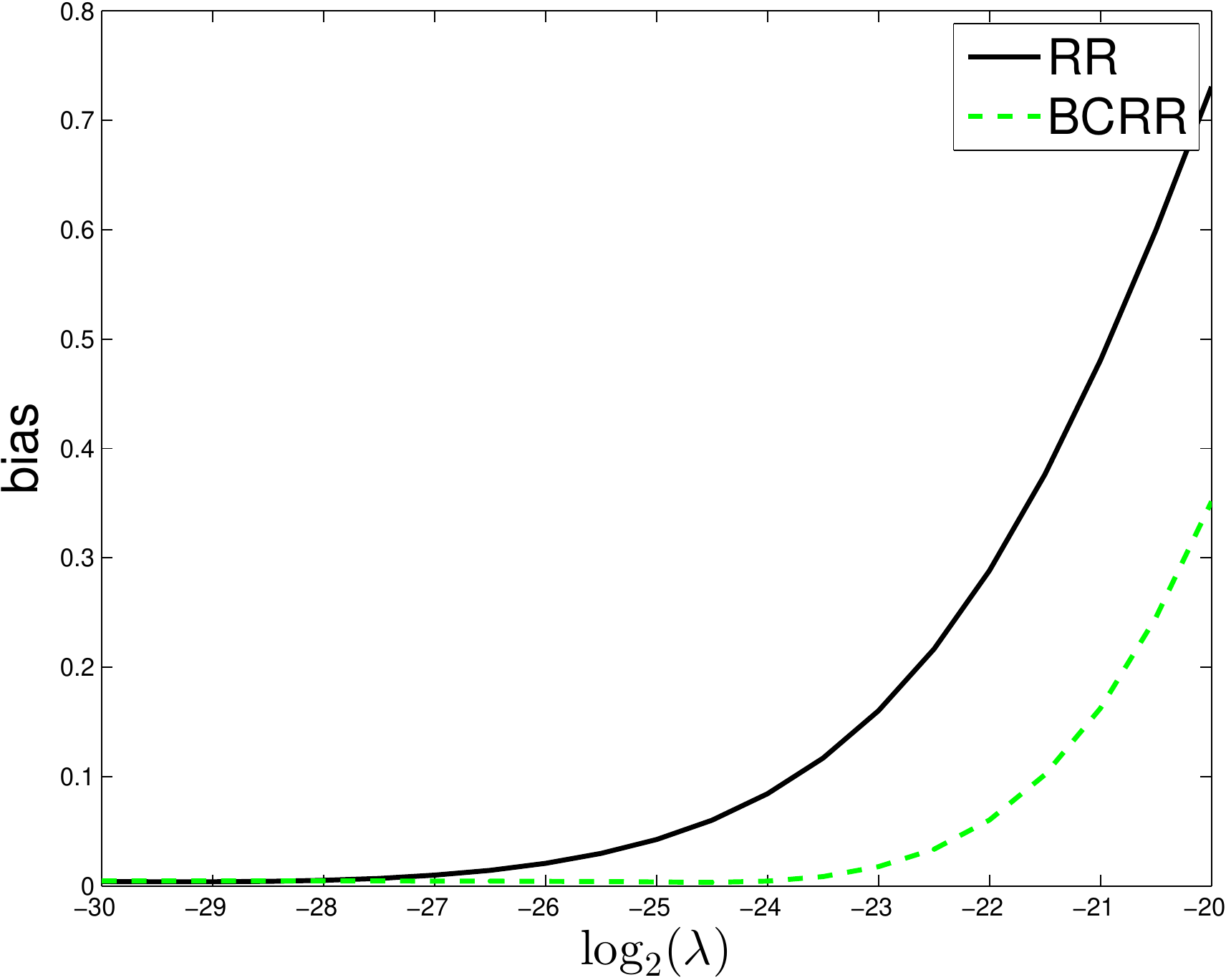}
\quad \includegraphics[width=0.32\textwidth]{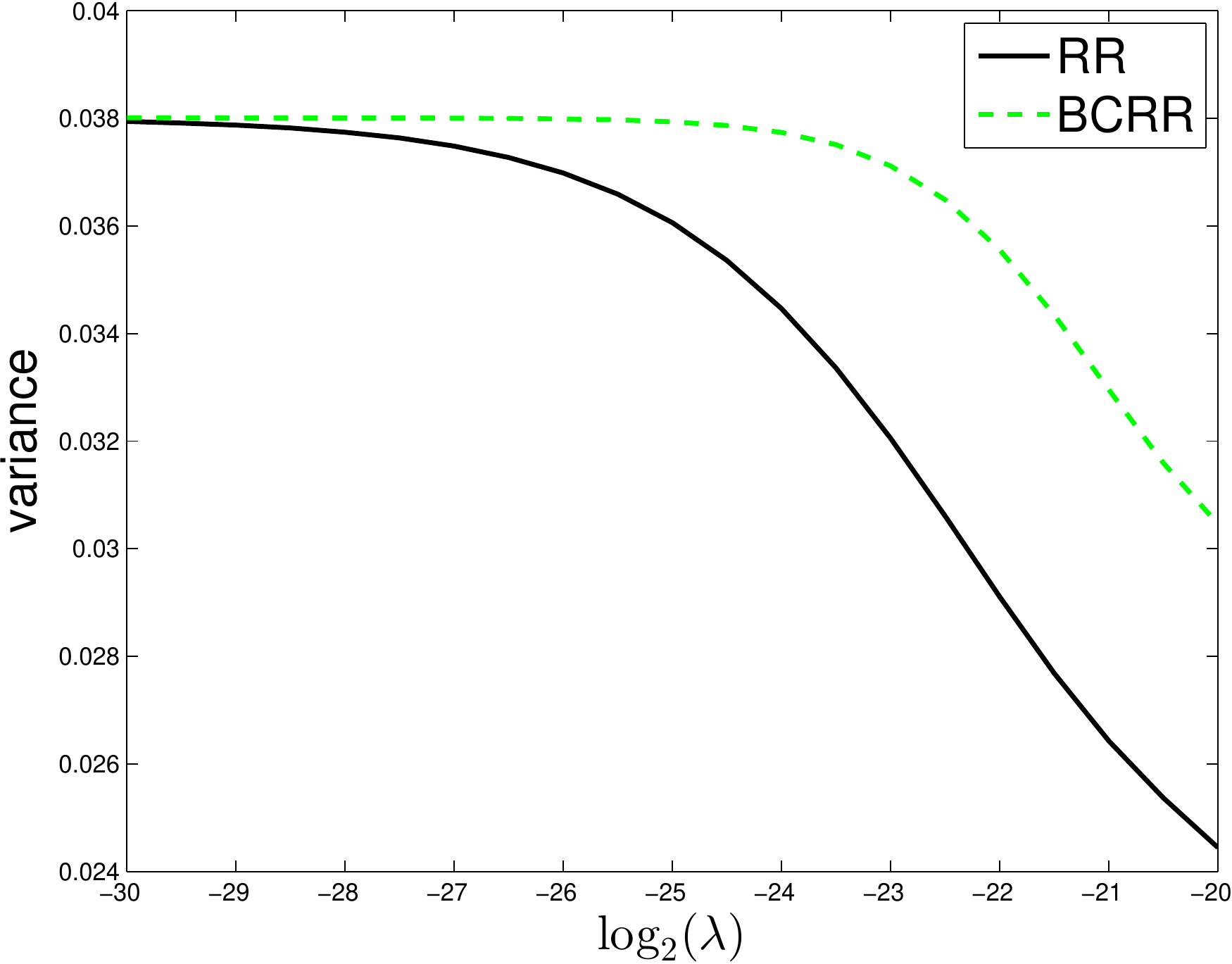}
\quad \includegraphics[width=0.32\textwidth]{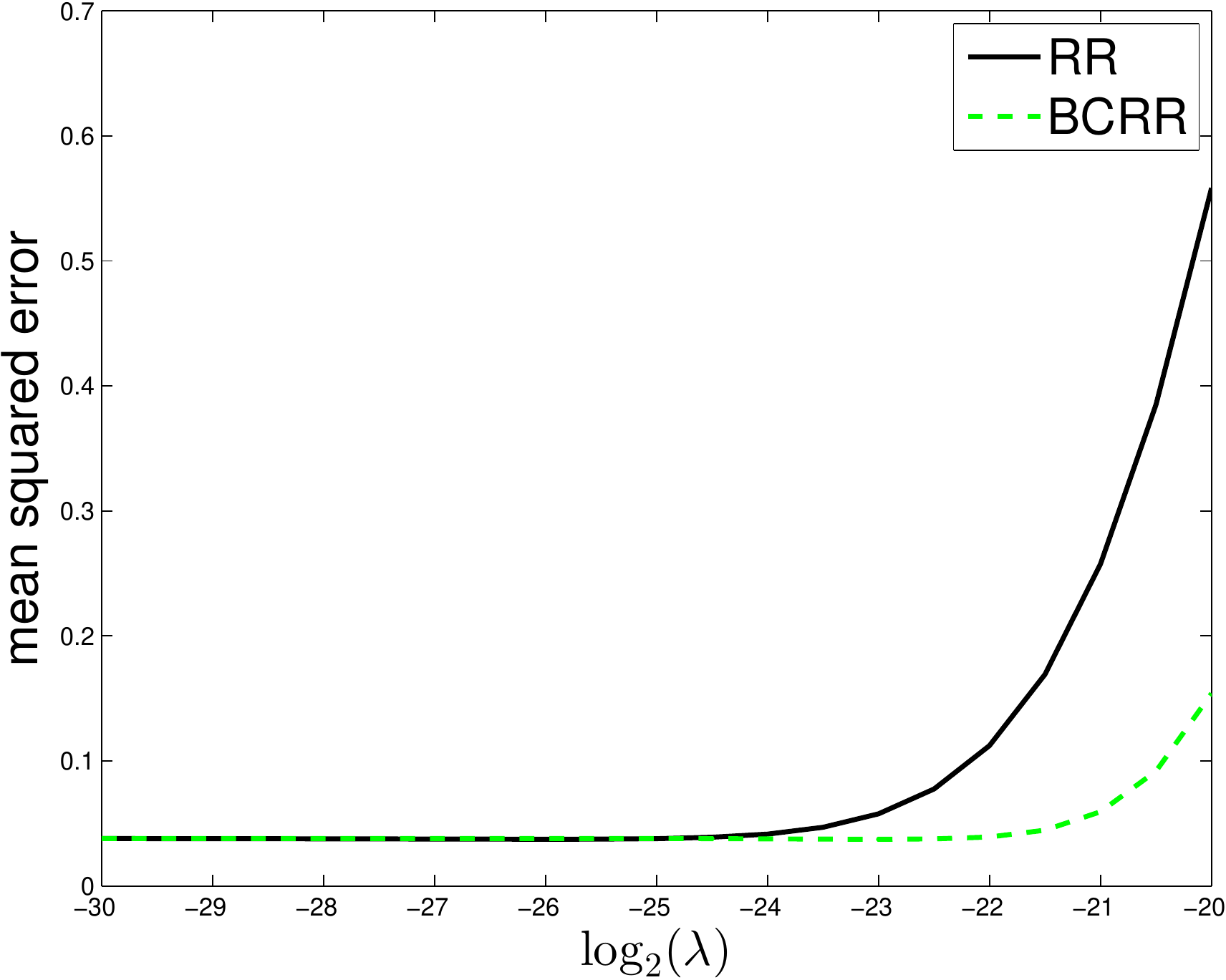}}
\caption{Model 2: Bias, variance, and mean squared error
of RR and BCRR.
\label{fig:bias2}}
\end{figure}

\subsection{Incremental learning with block wise streaming data}
\label{sec:sim1}

Now we compare the performance of RR and BARR
when they serve as base algorithms
 in incremental learning with block wise streaming data.
Again we use the two linear models above.
The experiment setting is also the same as before except that
20 data sets are generated in time to simulate the block streaming data.
When each data block comes,
10-fold cross validation is used to select the regularization parameter $\lambda$ for ridge regression.
Then both RR and BCRR use the same $\lambda$ to estimate the base model for the current data block.
The average of all available base models is then used for prediction.
As the prediction accuracy is the main concern, we use the mean squared prediction error to measure the learning performance.
After repeating the experiment 1000 times,
the mean squared prediction error is plotted in Figure \ref{fig:onlinemse}.
For Model 1, when $t$ is small, the performance of RR and BCRR are similar. As time goes on and more  data blocks come in,
the variance drops and the bias becomes the dominating term impacting the learning performance.
BCRR significantly outperforms RR.
For Model 2,  the model depends on the last four principal components.
With the optimal choice of $\lambda$,
$\frac \lambda{\lambda+\sigma_i}$ are close to 1 for $i=17, 18, 19, 20.$
We see bias reduction still helps but the improvement is not significant.

\begin{figure}[htbp]
\centerline{\includegraphics[width=0.45\textwidth]{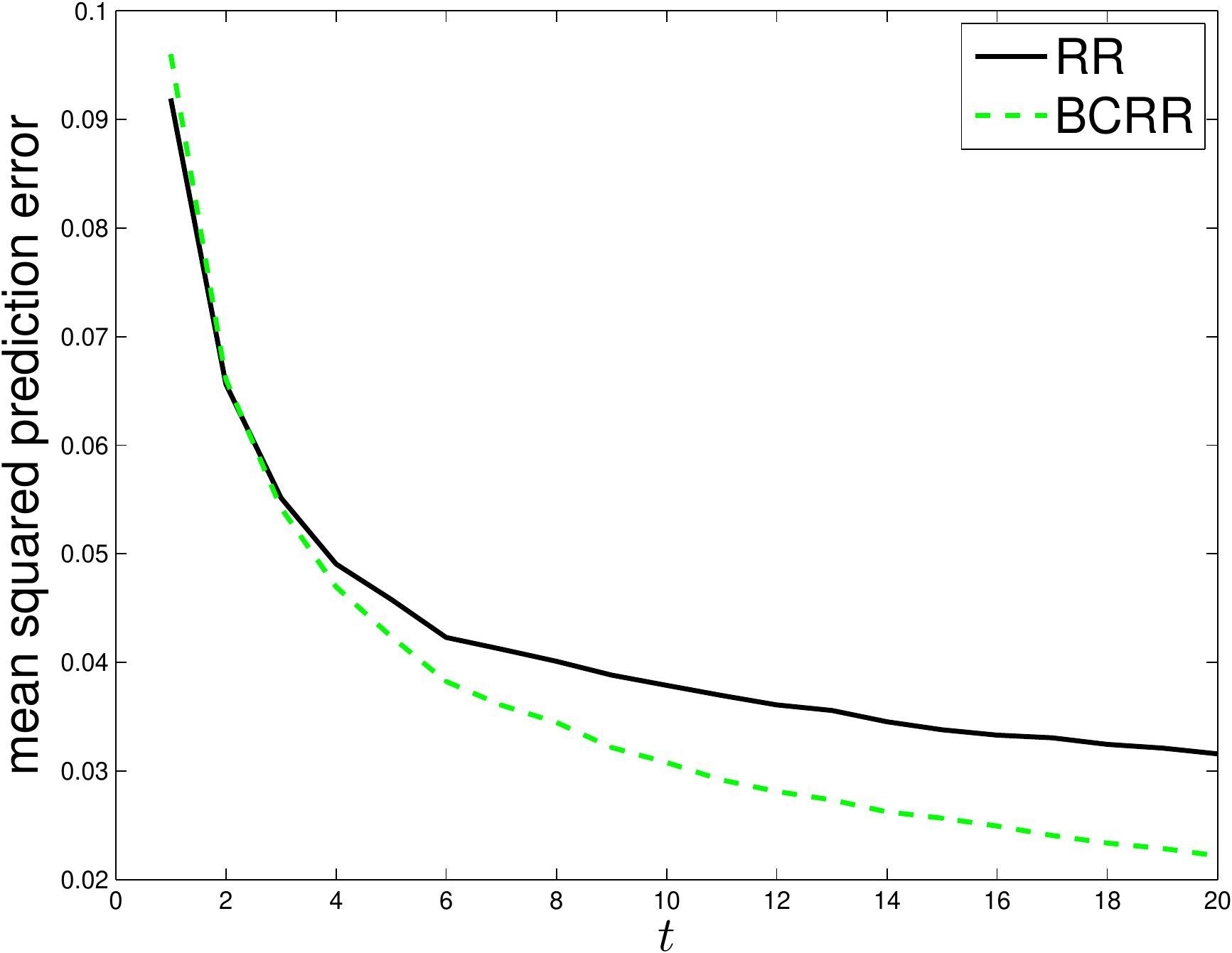}
\qquad\includegraphics[width=0.45\textwidth]{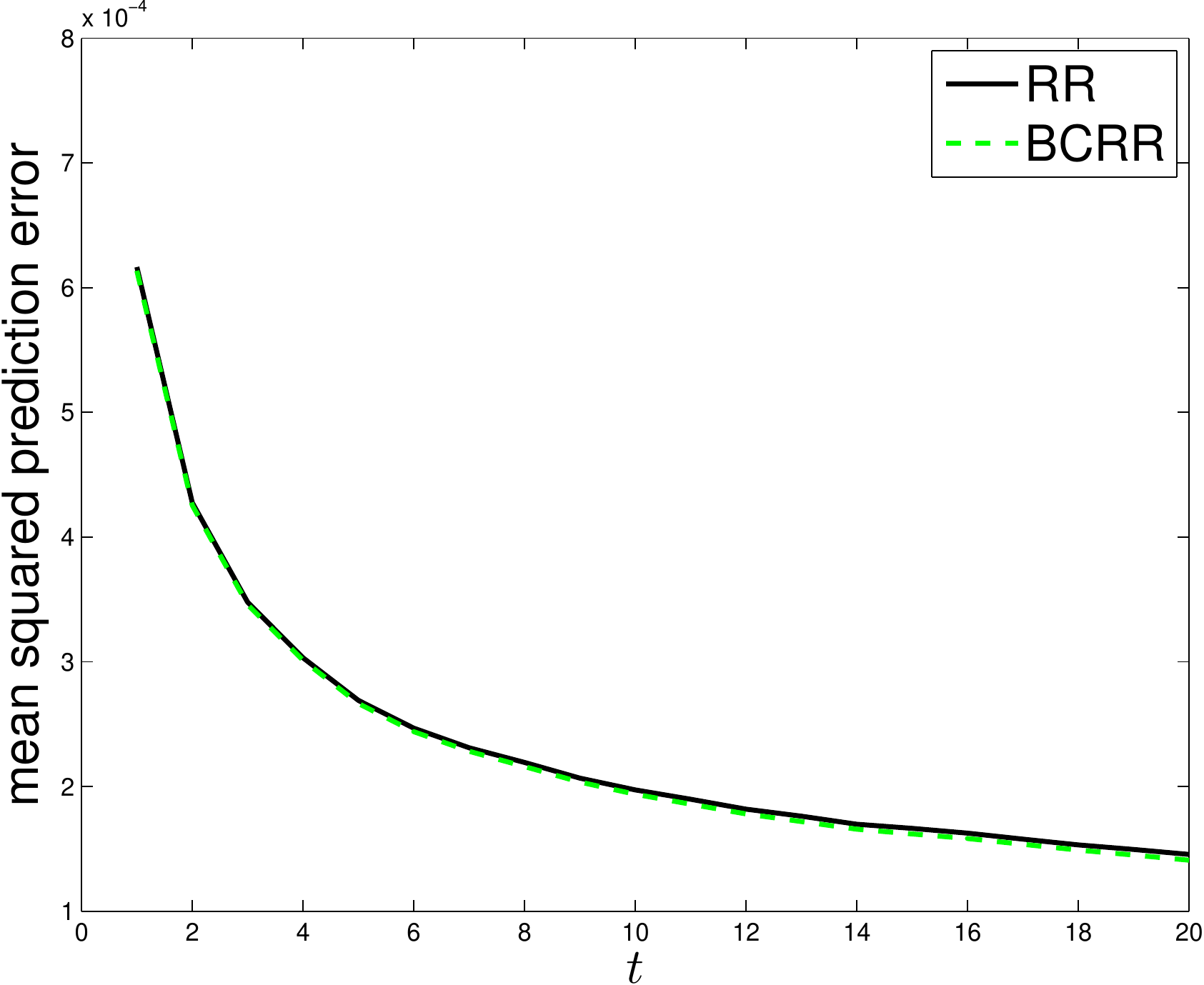}}
\centerline{\hfill (a) Model 1\hfill\hfill (b) Model 2\hfill}
\caption{Performance of incremental learning with  RR and BARR as the base algorithms. 
\label{fig:onlinemse}}
\end{figure}

\subsection{Real data}
\label{sec:real}

To further validate the effectiveness of BCRR,
two real-world data sets, the spam data and magic data,
 from UCI machine learning repository are
employed for empirical study in this research.

We following the procedure in \cite{he2011incremental}.
Each data set is  initially randomly sliced
into 20 chucks with identical size. At each run, one chuck is
randomly selected to be the testing data, and the remaining
19 chucks are assumed to come in sequentially.
Simulation results for each data sets are averaged across 20 runs.
As both problems are classification problems,
both the mean squared prediction error and classification accuracy are used
to measure the learning performance.

For spam data, the results are shown in Figure \ref{fig:spam}. For magic data,
the results are shown in Figure \ref{fig:magic}.
For both data sets, BCRR is more effective than ridge regression.
Note small mean squared prediction error usually implies small classification error.
However, such a relationship is not exact. That is why the classification errors fluctuate
for magic data, although the mean squared error decreases stably.

\begin{figure}[ht]
\centerline{\includegraphics[width=0.45\textwidth]{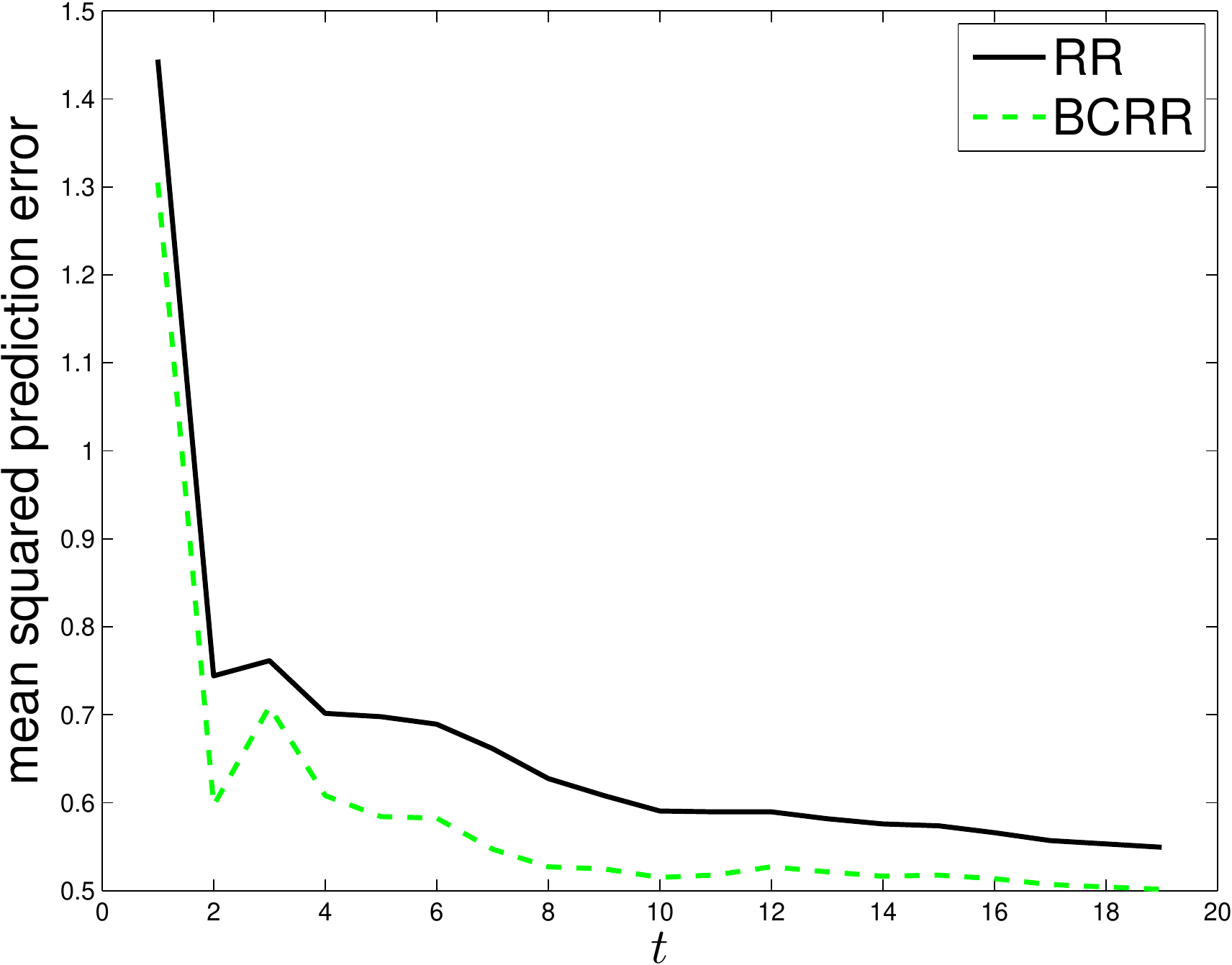}
\qquad \includegraphics[width=0.45\textwidth]{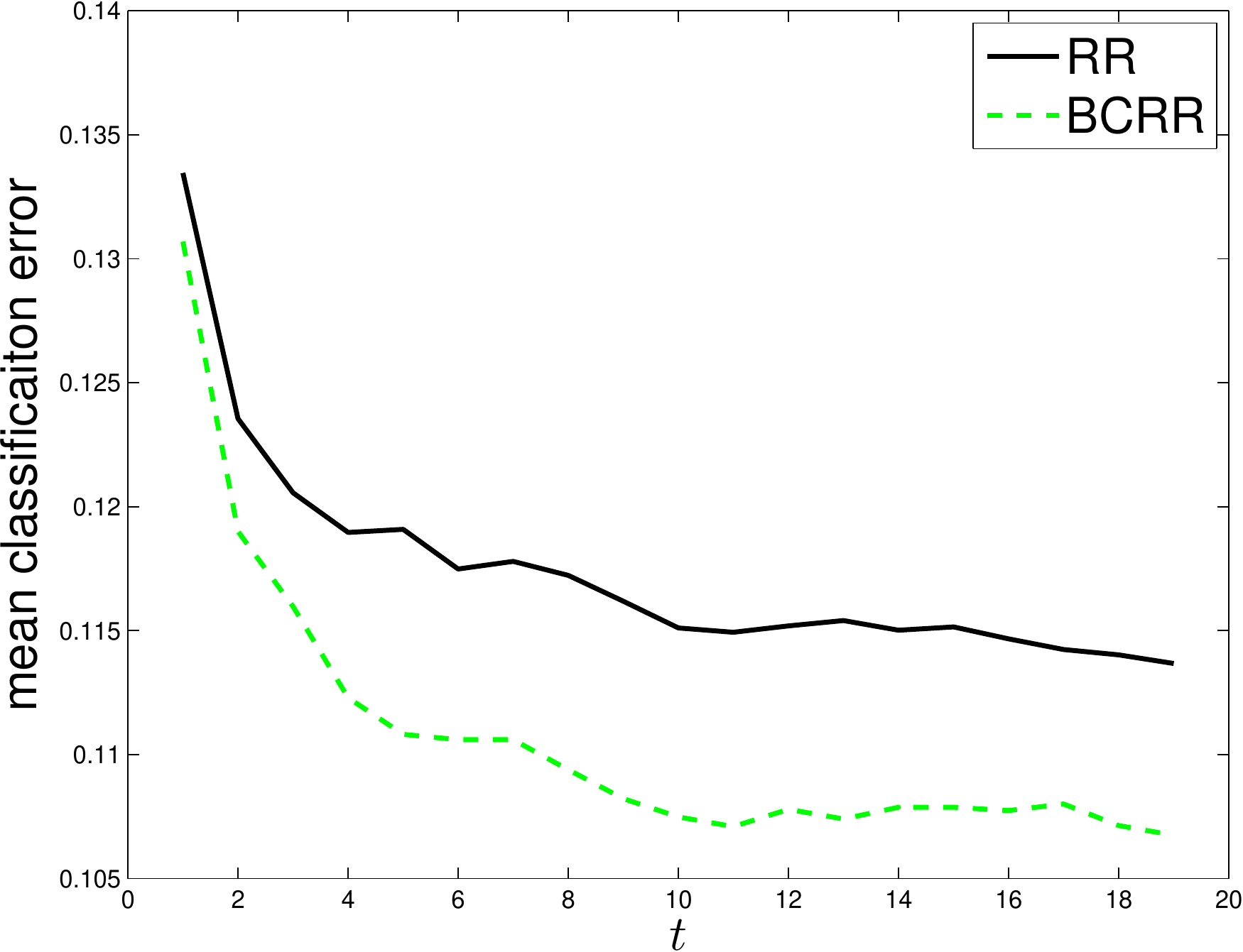}}
\caption{Spam data: mean squared prediction error and mean classification error on  test set.
\label{fig:spam}}
\end{figure}
\begin{figure}[htbp]
\centerline{\includegraphics[width=0.45\textwidth]{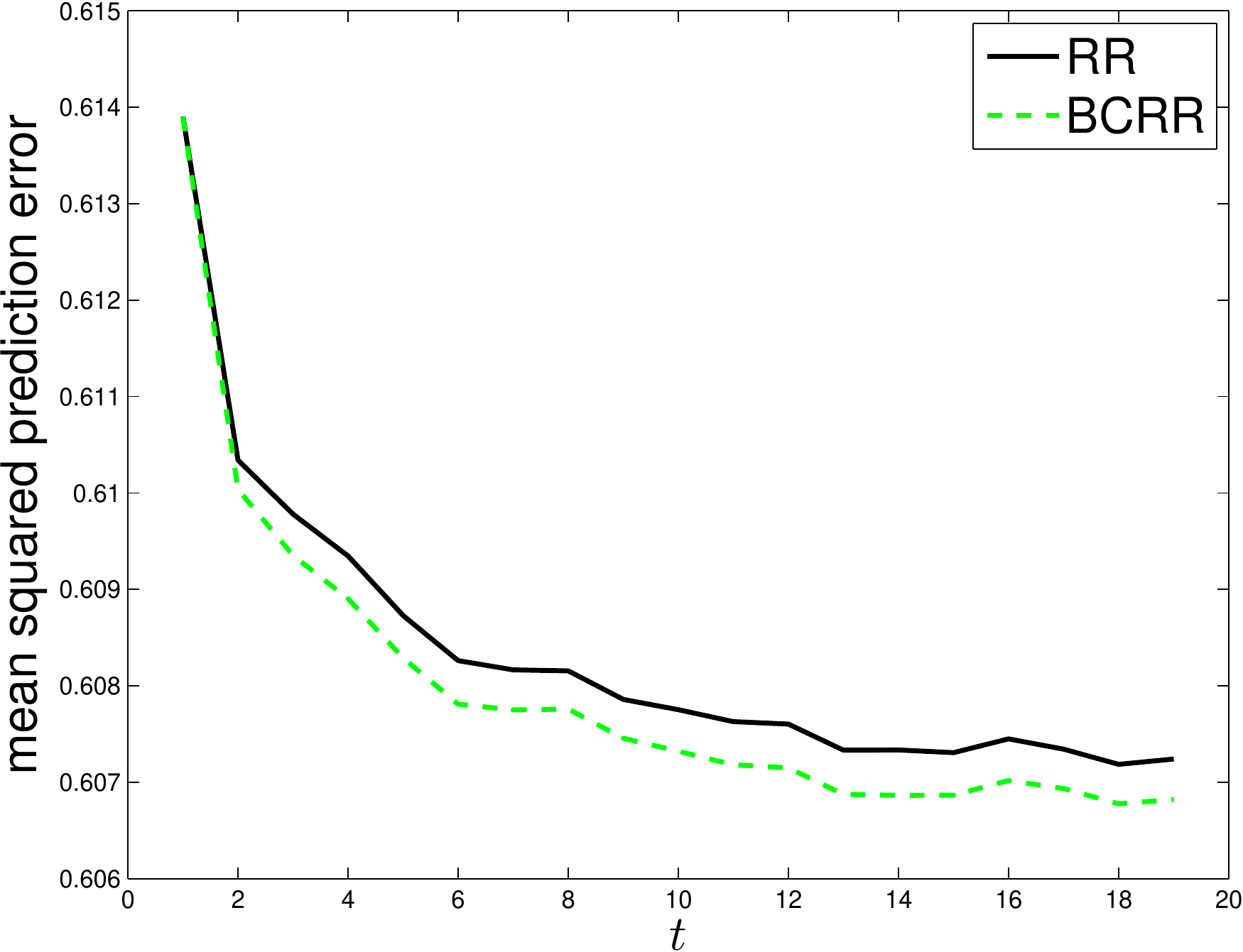}
\qquad \includegraphics[width=0.45\textwidth]{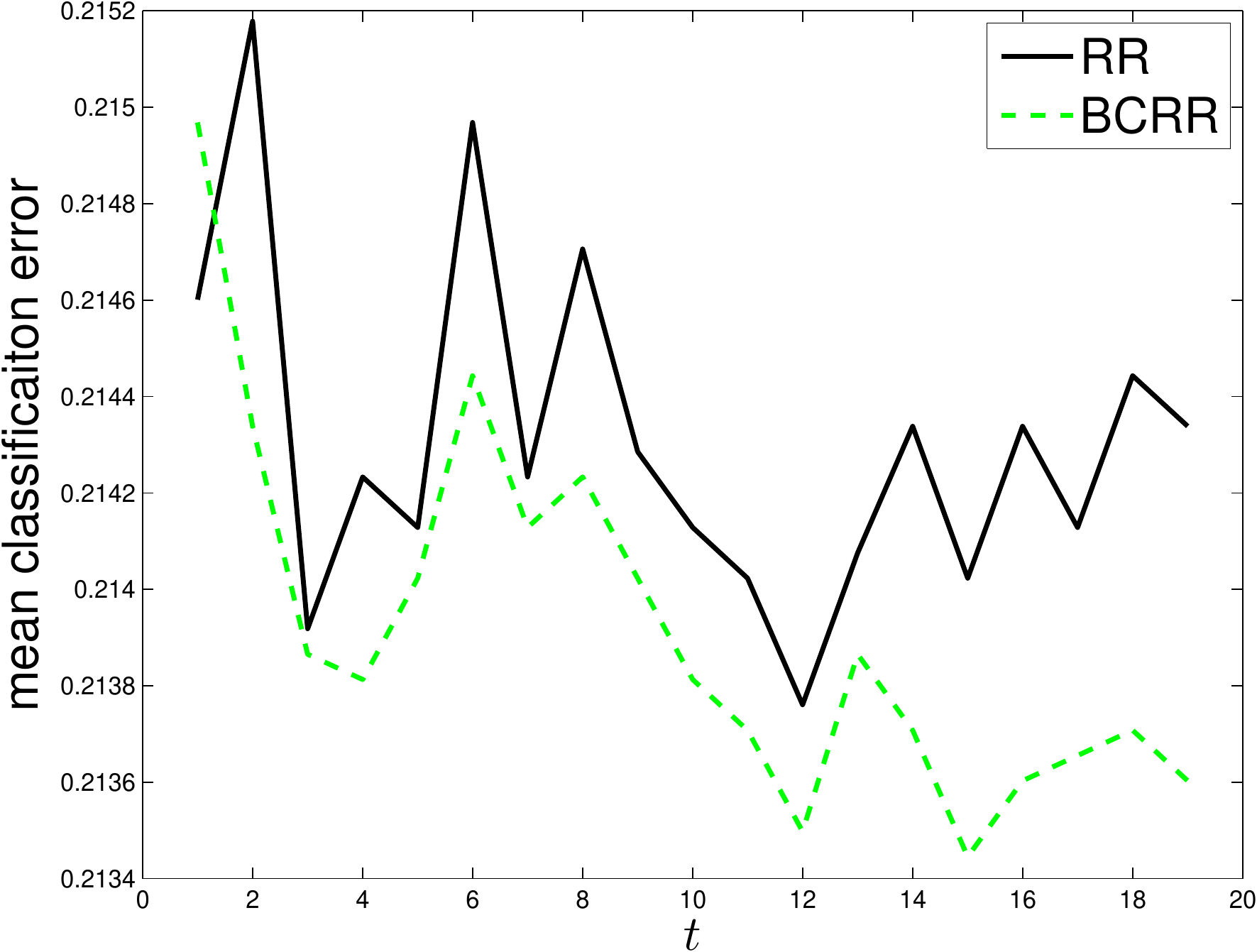}}
\caption{Magic data: the mean squared  error and mean classification error on test set.
\label{fig:magic}}
\end{figure}

\subsection{Kernel models}

Kernel methods are effective to handle the data
that contain strong nonlinear structure.
We applied the regularization kernel network (RKN) and
the bias corrected regularization kernel network (BCRKN)
to the handwritten digits recognition and compare their performance.

The MNIST handwritten digits data \cite{LeCun} 
is believed to have strong nonlinear structures
and has been a benchmark to
test the performance of nonlinear algorithms.
It contains 60,000 images of handwritten digits {0, 1, 2, $\ldots$ , 9}
as training data and 10,000 images as test data.
Each image consists of $p = 28 \times  28 = 784$
 gray-scale pixel intensities.
 The digits 2, 3, 5, 8 are considered to be
 most difficult to recognize and
 nonlinear models are able to help.

In our analysis, we consider the classification of  digit 3 versus digit 8.
Note our purpose is to compare the
performance of RKN and BCRKN and verify the
effectiveness of bias correction, not to
find the best classifier. So we select the
Gaussian kernel, set bandwidth parameter
as the median of the pairwise distance  between the images
and do not optimize it in the learning process.
The training data is sliced into 50
chucks with identical size to mimic the data stream.
The mean squared error and classification error
on the test data is reported in Figure \ref{fig:dig38k}.
It is clear that bias correction helps to improve learning performance.
\begin{figure}[htbp]
\centerline{\includegraphics[width=0.45\textwidth]{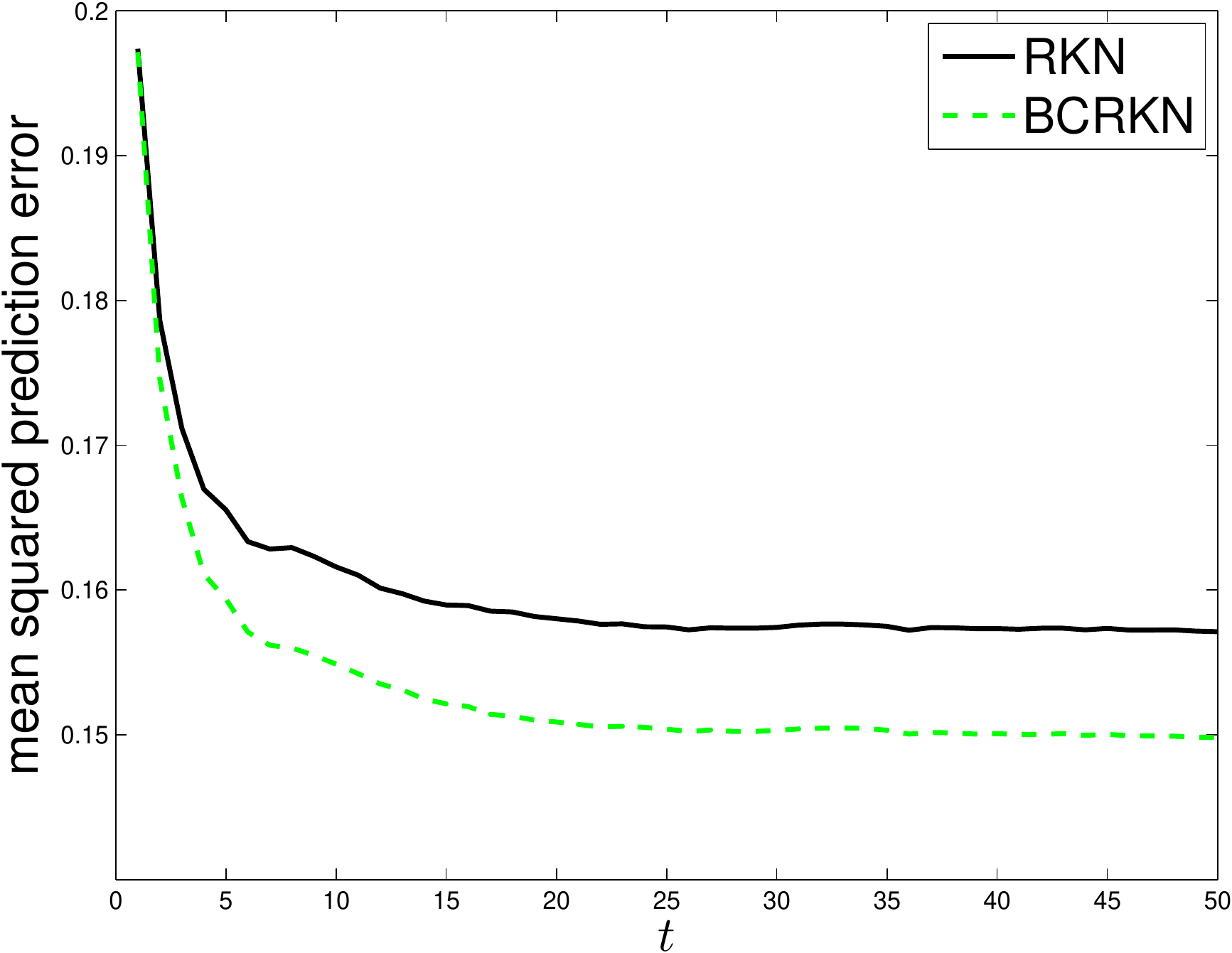}
\qquad \includegraphics[width=0.45\textwidth]{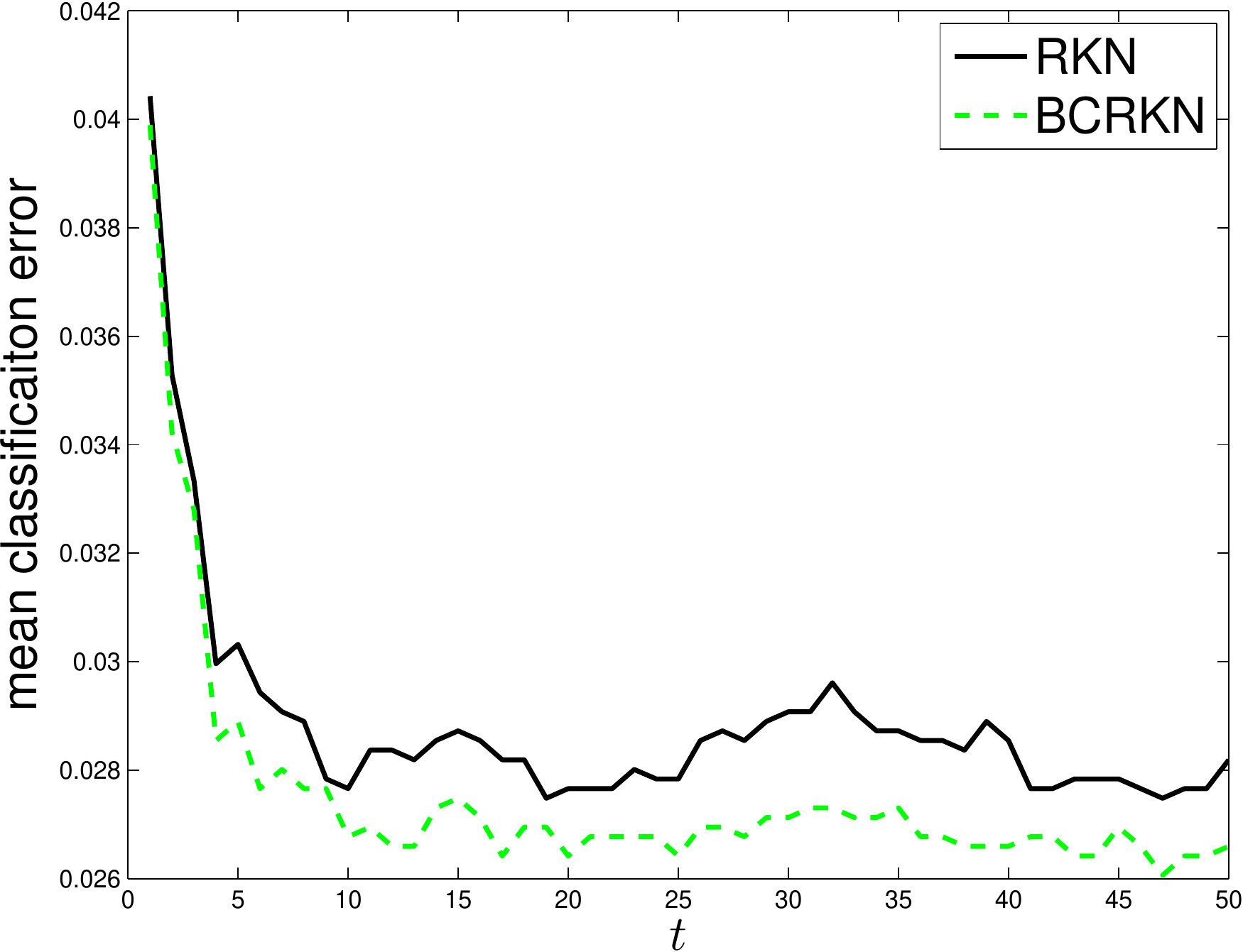}}
\caption{Classification of digit 3 versus digit 8 in the MNIST data.
\label{fig:dig38k}}
\end{figure}

 \section{Conclusions and discussions}
 \label{sec:conclusion}

 We proposed two new regularized regression algorithms
 that are derived by correcting the bias of 
 ridge regression and regularization kernel network.
 The bias corrected algorithms are shown to 
 have smaller bias, while, in general, they
 have slightly larger variance. 
 When applied to a single data set,
 the bias corrected algorithms have comparable performance 
 to the uncorrected algorithms.
 But the smaller bias favors their use 
 in learning with block wise data, such as
 incremental learning with block streaming data 
 or the divide and conquer algorithm.

Bias correction is found less effective
when the true model depends \emph{only} on
the principal components with very small eigenvalues.
Fortunately, this is not common in real applications.
When the  true model does depend heavily on the
principal components with small eigenvalues,
the bias correction performs similar to
the uncorrected algorithm, not worse.
Furthermore, the bias corrected algorithms may be more
computationally stable because it allows using
larger regularization parameter to achieve the same
learning performance.
Therefore, the use of bias corrected algorithms
in practice is safe and preferable.

It is natural to consider the possibility and necessity 
 of higher order bias correction. We remark that 
 defining higher order bias corrected estimators is possible.
 But it is unnecessary because higher order bias correction is ineffective.
 To illustrate this, consider the ridge regression. 
 We know from Section \ref{sec:barr}
 that the asymptotic bias of $\hat\bw^\sharp$ is
 $-\lambda^2(\lambda I + \Sigma)^{-2} \bw.$
 We can subtract an estimate of this asymptotic bias
 to obtain a second order bias corrected ridge regression
 estimator
$$\hat \bw ^\sharp_{2} = \hat \bw^\sharp
+ \lambda^2(\lambda I + \hat\Sigma)^{-2}\hat\bw.$$
This estimator will have an asymptotic bias
$-\lambda^3(\lambda I + \Sigma)^{-3} \bw,$
which can be used to define the
third order bias corrected estimator.
Actually, for any $k>2$, we can define
bias corrected estimators of order $k$ iteratively, 
$$\hat \bw^\sharp_k = \hat\bw^\sharp_{k-1} +
\lambda^k(\lambda I +\hat\Sigma)^{-k} \hat \bw,$$
which has an asymptotic bias
$-\lambda^{k+1}(\lambda I +\Sigma)^{-(k+1)} \bw.$
We can also prove that the asymptotic bias
decreases as $k$ increases.
However, simulations show higher order
bias correction is ineffective.
As an illustration, we applied the
bias corrected estimators of order 1, 2, and 3
to the streaming data generated for
Model 1 in Section \ref{sec:sim1}.
We see from Figure \ref{fig:model1h} that
the performance of higher order bias correction
is even worse.
\begin{figure}[htbp]
\centerline{\includegraphics[width=0.6\textwidth]{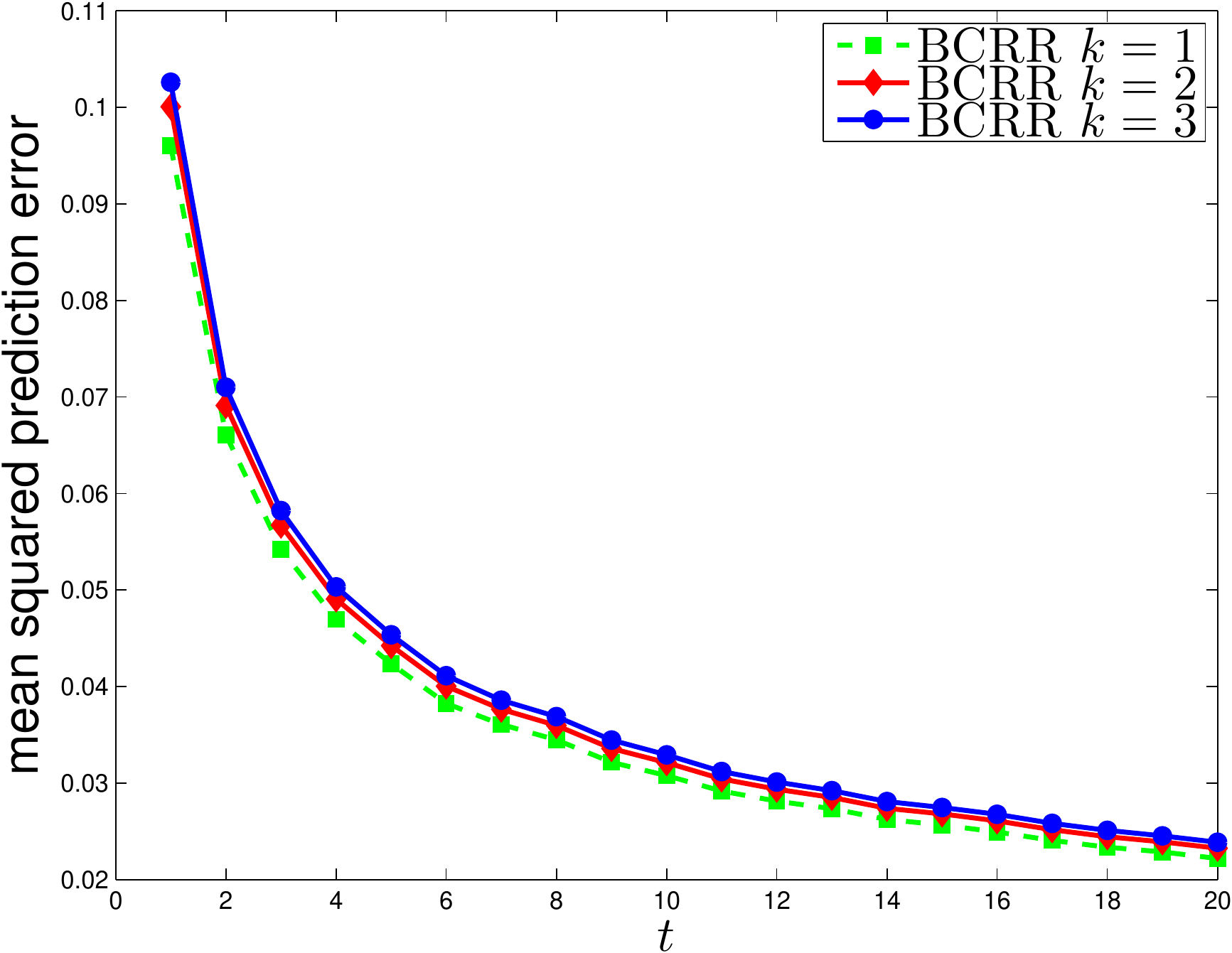}}
\caption{Model 1: Performance of BCRR of order 1, 2, and 3.
\label{fig:model1h}}
\end{figure}

\bibliographystyle{abbrv}
\bibliography{barr}

\end{document}